%% file: technical-report.tex
\documentclass{article}
\usepackage{arxiv}

\usepackage[utf8]{inputenc} 
\usepackage[T1]{fontenc}    
\usepackage{hyperref}       
\usepackage{url}            
\usepackage{booktabs}       
\usepackage{amsfonts}       
\usepackage{nicefrac}       
\usepackage{microtype}      
\usepackage{lipsum}
\usepackage{todonotes}
\usepackage{subcaption}

\title{Exact and Metaheuristic Approaches for the Production Leveling Problem}

\input{definitions.tex}

\author{
  Johannes Vass \\
  \texttt{jvass@dbai.tuwien.ac.at}\\
  \And
  Marie-Louise Lackner\\
  \texttt{marie-louise.lackner@tuwien.ac.at} \\
  \And
  Nysret Musliu \\
  \texttt{musliu@dbai.tuwien.ac.at}\\[3em]
  Christian Doppler Laboratory for Artificial Intelligence and Optimization for Planning and Scheduling \\
  Databases and Artificial Intelligence Group \\
  TU Wien, 1040 Wien, Austria
}

\begin{document}

\maketitle

\begin{abstract}
In this paper we introduce a new problem in the field of production planning which we call the Production Leveling Problem.
The task is to assign orders to production periods such that the load in each period and on each production resource is balanced, capacity limits are not exceeded and the orders' priorities are taken into account.
Production Leveling is an important intermediate step between long-term planning and the final scheduling of orders within a production period, as it is responsible for selecting good subsets of orders to be scheduled within each period.

A formal model of the problem is proposed and \NP-hardness is shown by reduction from Bin Backing.
As an exact method for solving moderately sized instances we introduce a MIP formulation.
For solving large problem instances, metaheuristic local search is investigated.
A greedy heuristic and two neighborhood structures for local search are proposed, in order to apply them using Variable Neighborhood Descent and Simulated Annealing.
Regarding exact techniques, the main question of research is, up to which size instances are solvable within a fixed amount of time.
For the metaheuristic approaches the aim is to show that they produce near-optimal solutions for smaller instances, but also scale well to very large instances.

A set of realistic problem instances from an industrial partner is contributed to the literature, as well as random instance generators.
The experimental evaluation conveys that the proposed MIP model works well for instances with up to 250 orders.
Out of the investigated metaheuristic approaches, Simulated Annealing achieves the best results.
It is shown to produce solutions with less than $3\%$ average optimality gap on small instances and to scale well up to thousands of orders and dozens of periods and products.
The presented metaheuristic methods are already being used in the industry.
\end{abstract}


\section{Introduction}

Production systems are being subject to continuous and radical change in the course of the last decades.
The need for productivity improvements is provoking companies to invest heavily in automation on all levels.
Production planning plays a major role in these developments as the replacement of manual planning with software-assisted or even autonomous systems can lead to considerable efficiency increases.

We introduce a new problem in the field of production planning which arises from the needs of an industrial partner.
It is a combinatorial optimization problem which treats the leveling of production and we therefore call it the \gls{plp}.
It belongs to medium-term planning, which means it is intended to be embedded between long-term planning and the scheduling of the concrete production sequence.

The problem is concerned with assigning orders of certain product types and demand sizes to production periods such that the production volume of each product type is leveled across all periods.
Furthermore, the overall amount produced in each period is subject to leveling as well.
A solution is feasible if the production volumes to be leveled do not exceed given maximum values.
The optimization part consists in minimizing the deviation of the production from the optimal balance, while at the same time making sure that the orders are assigned approximately in the order of their priorities.
The idea behind this goal is that considering orders only in the order of decreasing priority, as it is often done, frequently leads to spikes and idle times for certain resources involved in the production process.
Leveling these highs and lows results in a smoother production process because a similar product mix is produced in every period.
It is important to note, that the solution to the \gls{plp} is not a schedule since the orders are only assigned to production periods but the concrete execution sequence and assignment to machines and workers is not part of this problem.
The intention is rather so serve as a step between long-term planning and short-term production scheduling.

The main goal of this work consists of modeling the \gls{plp} formally and developing solution strategies for it.
The primary method we propose is Simulated Annealing because of its capability to solve large-scale instances of the problem.
Therefore, we developed two move types for obtaining neighboring solutions, which are used in Simulated Annealing.
Furthermore, we investigate \gls{vnd}, which is a deterministic algorithm using also these two neighborhoods.
We also propose a \gls{mip} model in order to obtain optimal solutions and lower bounds.
We are interested in finding the border between instances which can be solved exactly and those where the exponential nature of the problem makes the usage of \gls{mip} impractical.
Furthermore, we want to investigate how near we can get to optimality by the local search methods we propose.
To sum up, the main contributions of this paper are:
\begin{itemize}
	\item A mathematical model for the \gls{plp},
	\item a proof of \NP-hardness,
	\item a \gls{mip} model,
	\item two neighborhood structures for local search,
	\item realistic and randomly generated problem instances and
	\item an extensive evaluation of MIP and metaheuristic methods.
\end{itemize}

The rest of this paper is structured as follows:
Section~\ref{sec:problem-statement} presents the problem first informally, then a precise mathematical formulation is given and finally related work is discussed.
In Section~\ref{sec:complexity} the theoretical complexity of the \gls{plp} is analyzed, yielding an \NP-hardness result.
Section~\ref{sec:MIP} introduces a \gls{mip} model.
In Section~\ref{sec:local_search} we finally turn to local search methods and present the variant of Simulated Annealing we apply.
Experimental evaluation is performed in Section~\ref{sec:evaluation}, providing answers to the research questions which we formulated above.
Section~\ref{sec:conclusion} summarizes the results and presents ideas for future work.

\section{Problem Statement and Related Work}
\label{sec:problem-statement}

In this section we first want to build up an intuition for the \gls{plp} which we approach through the use of examples.
Then we specify the parameters, constraints and objective function formally.
Finally we shed light on related problems presented in the literature.

\subsection{Problem Description}
The input to the \gls{plp} are a list of orders, each of them having a demand value, priority and product type.
Furthermore, we are given a set of periods and the maximum production capacity per period, both for all product types together and for each one separately.
We search for solutions by finding an assignment of orders to periods such that the production volume is balanced between the periods while trying to stick to the sequence implied by the order's priorities as good as possible.
We can see the objective function of the \gls{plp} as the task of finding a good tradeoff between the following goals:
\begin{enumerate}
	\item Minimize the sum of deviations of the planned production volume to the average demand (i.e.\ the target value) for each period, ignoring the product types. This makes sure that the overall production per period is being leveled.
	\item Minimize the sum of deviations of the production volume of each product type to its respective mean (target) value, making sure that the production of each product type is being leveled.
	\item Minimize the number of times a higher prioritized order is planned for a later period than a lower prioritized order, which we call a priority inversion. This objective makes sure that more important orders are scheduled in earlier periods.
\end{enumerate}

Let us now explore the three optimization goals by means of examples:
\begin{enumerate}
	\item Figure~\ref{fig:plp-example-o1} shows a tiny example instance with five orders, which are shown as boxes, where the box height corresponds to the order size.
	The orders should be assigned to the three periods such that the distances of the stacks of orders and the dashed target line is minimized and no stack crosses the red line which represents the capacity limit.
	It is easy to see that this solution is optimal w.r.t. the leveling objective.
	
	\item Figure~\ref{fig:plp-example-o2} shows a slightly larger problem instance which has three product types (blue, green and red).
	It visualizes the solution from the perspective of the second objective, which works the same way as the first but discriminates by product types.
	That is, we seek to minimize for each of the three product types blue, green and red the deviation from the dashed target line in each period.
	
	\item Figure~\ref{fig:plp-example-o3} comes back to the first example, but views it from the perspective of the third objective.
	The numbers inside the orders signalize the priorities, whereby a larger number indicates a higher importance.
	That being said it is obvious that the red order should not be assigned to an earlier period than the yellow or the blue one.
	This bad state between red and blue/yellow is what we call a priority inversion and their number is what should be minimized by the third goal.
	In the example, a better solution can be easily constructed for example by swapping the red order with the yellow one, because it would make both priority inversions disappear.
\end{enumerate}

\begin{figure}
	\centering
	\includegraphics[width=.5\textwidth, trim={0 0 0 .9cm},clip]{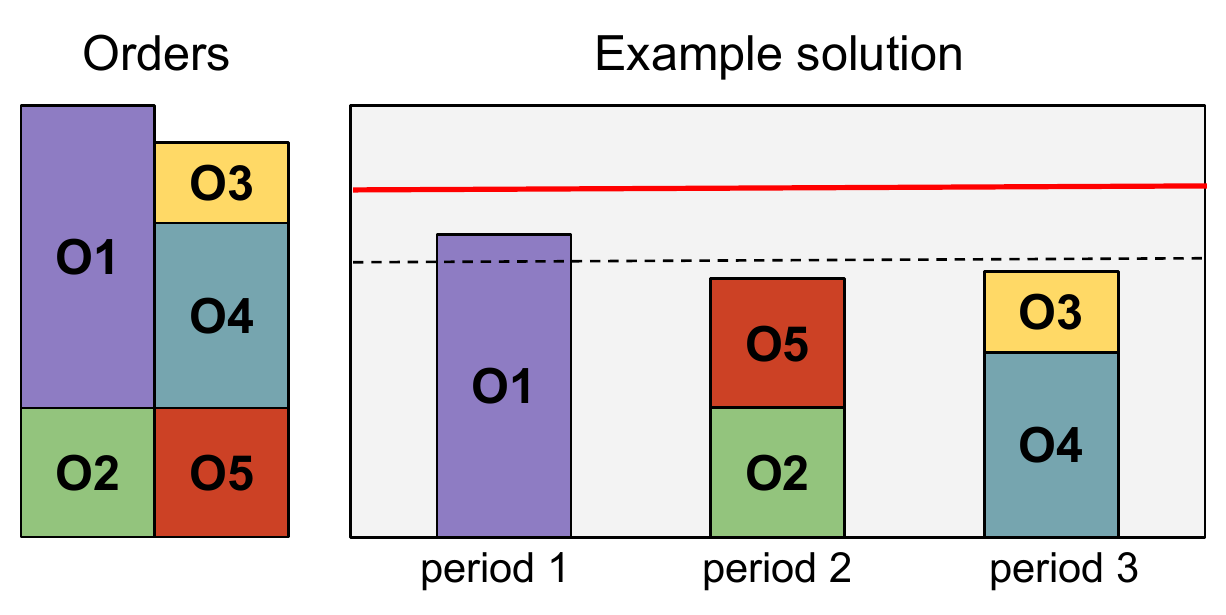}
	\caption{Example of the leveling objective for the total production amount. The dashed line represents the target value, the solid red line the capacity limit. This is the optimal solution w.r.t.\ the leveling objective.}
	\label{fig:plp-example-o1}
\end{figure}	

\begin{figure}
	\centering
	\includegraphics[width=.4\textwidth, trim={13.7cm 0.3cm 12.6cm .4cm},clip]{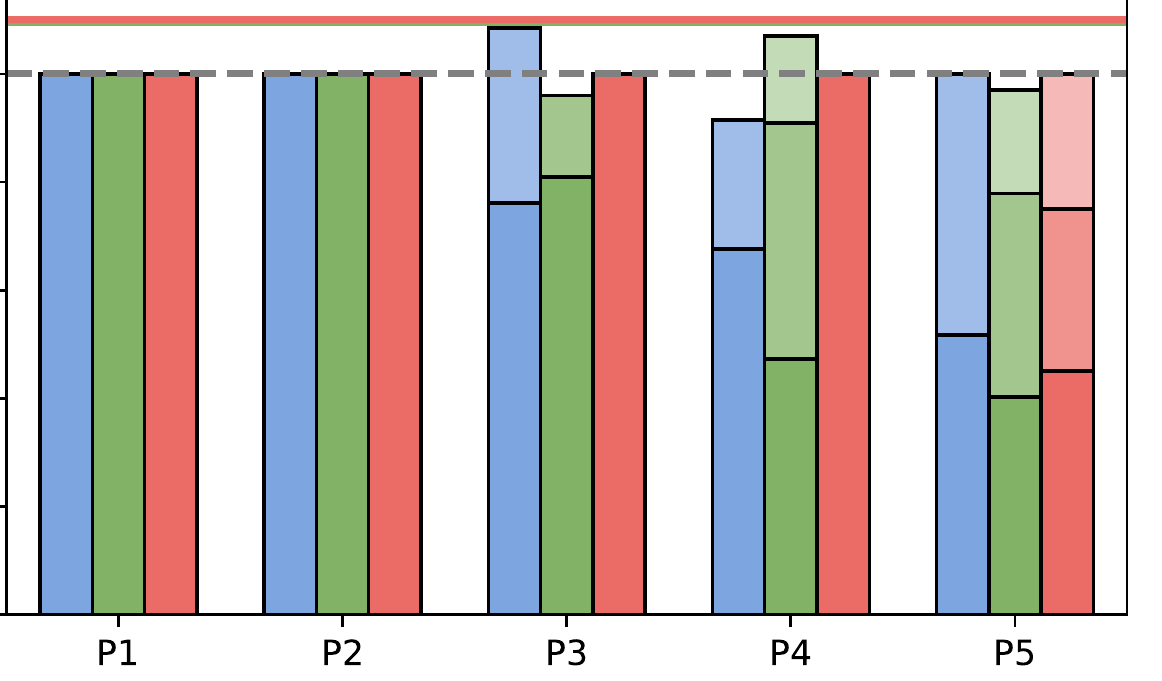}
	\caption{Example of the leveling objective for each product type (blue, green, red).}
	\label{fig:plp-example-o2}
\end{figure}

\begin{figure}
	\centering
	\includegraphics[width=.5\textwidth, trim={0 0 0 .9cm},clip]{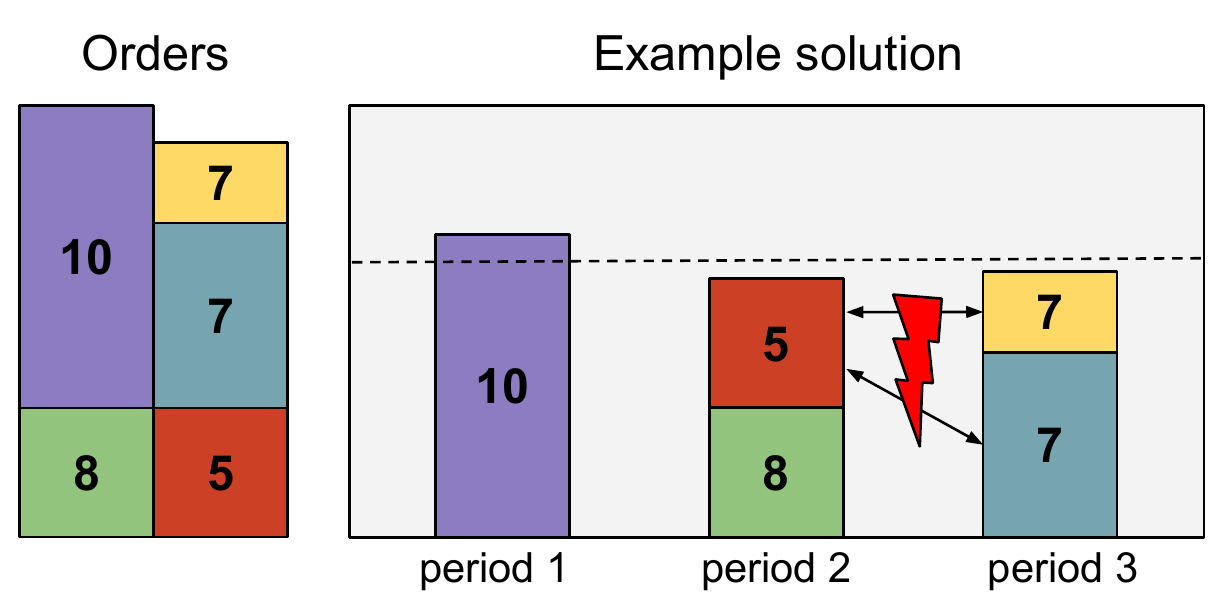}
	\caption{The same solution as above has two priority inversions. An optimal solution w.r.t.\ priorities would be to swap the red and the yellow order.}
	\label{fig:plp-example-o3}
\end{figure}

We have seen in the examples, that an optimal solution w.r.t.\ one objective is not necessarily optimal w.r.t.\ another.
As we want to combine the three objectives into one by a weighted sum, the location of the optima will clearly depend on the weights.
These weights must be determined on the basis of a specific use case because there is no general way to decide without domain knowledge e.g.\ how important one priority inversion is compared to one unit more of imbalance.
We worked out a sensible default weighting in cooperation with our industrial partner based on their real-life data and use it for all experiments throughout the paper.
How the objective function can be formally stated and how the weighting works is described in more detail in the following section.

\subsection{Mathematical Formulation}
\label{sec:mathematical-formulation}


Now we turn towards a formal description of the the problem, consisting of parameters, variables, constraints and the objective function.

\subsection*{Input parameters}
\small
\begin{tabular}{p{2cm}l}
	$K \subseteq  \mathbb{Z}^+$	& Set of orders $\{i \in \mathbb{Z}^+ | 1 \leq i \leq k\}$, where $k$ is the number of orders\\
	$M \subseteq  \mathbb{Z}^+$	& Set of product types $\{i \in \mathbb{Z}^+ | 1 \leq i \leq m\}$, where $m$ is the number of\\
	& product types\\
	$N \subseteq  \mathbb{Z}^+$	& Set of periods $\{i \in \mathbb{Z}^+ | 1 \leq i \leq n\}$, where $n$ is the number of periods\\
	$a_i \in \mathbb{R}^+$		& for each objective function component $i \in \{1, 2, 3\}$ the associated weight\\
	$c \in \mathbb{R}^+$       & the maximum overall production volume per period\\
	$c_t \in \mathbb{R}^+$     & for each product type $t \in M$ the maximum production volume per period\\
	$d_j \in \mathbb{Z}^+$     & for each order $j \in K$ its associated demand\\
	$p_j \in \mathbb{Z}^+$     & for each order $j \in K$ its associated priority\\
	$t_j \in \mathbb{Z}^+$     & for each order $j \in K$ the product type\\
	$d^* \in \mathbb{Z}^+$     & the target production volume per period, i.e.\ $\frac{1}{n} \sum_{j \in K} d_j$ \\
	$d_t^* \in \mathbb{Z}^+$   & the target  production volume per period for each product type $t \in M$, \\
	& i.e.\ $\frac{1}{n} \sum_{j \in K | t_j = t} d_j$
\end{tabular}
\normalsize

\subsection*{Variables}
\begin{itemize}
	\item For each order the production period for which it is planned:
	$$y_j \in N \quad \forall j \in K$$
	\item The production volume for each period (helper variable):
	$$w_i = \sum_{\substack{j \in K: \\ y_j = i}} d_j \quad \forall i \in N$$
	\item The production volume for each product type and period (helper variable): 
	$$w_{i,t} = \sum_{\substack{j \in K: \\ y_j = i \land t_j = t}} d_j \quad \forall i \in N, \forall t \in M$$
\end{itemize}

\subsection*{Hard constraints}
\begin{itemize}
	\item The limit for the overall production volume is satisfied for each period:
	$$ \forall i \in N \quad  w_i \leq c $$
	\item The limit for the production volume of each product type is satisfied for each period:
	$$ \forall i \in N, t \in M \quad w_{i,t} \leq c_p $$
\end{itemize}

\subsection*{Objective function}

The following three objective functions represent the three targets to minimize:

\begin{align}
f_1 = & \sum_{i \in N} | d^* - w_i | \label{eq:f1}\\
f_2 = & \sum_{t \in M} \bigg(\frac{1}{d_t^*} \cdot \sum_{i \in N} | d^*_t - w_{i,t}|\bigg) \label{eq:f2}\\
f_3 = & \vert \left\lbrace (i, j) \in K^2 : y_i > y_j \text{ and } p_i > p_j \right\rbrace \vert \label{eq:f3}
\end{align}
Function $f_1$ represents the sum over all periods of deviations from the overall target production volume (i.e.\ all product types at once).
Function $f_2$ states the sum over all product types of sums over all periods of the deviations from the target production volume for that product type, normalized by the respective target value.
The normalization is done so that every product has the same influence onto the objective function regardless of whether its average demand is high or low.
Function $f_3$ counts the number of priority inversions in the assignment, or in other words the number of order-pairs $(i,j)$ for which $i$ is planned after $j$ even though $i$ has a higher priority than $j$.

In order to combine these three objectives into a single objective function and achieve a weighting which does not change its behavior between instances with different number of orders, periods or product types, the cost components need to be normalized.

\begin{align}
g_1 = & \frac{1}{n \cdot d^*} \cdot f_1 	\label{obj1_abs}\\
g_2 = & \frac{1}{n \cdot m}  \cdot f_2		\label{obj2_abs}\\
g_3 = & \frac{2}{k \cdot (k-1)} \cdot f_3	\label{obj3_abs}
\end{align}
The normalization ensures that $g_1$ and $g_2$ stay between 0 and 1 with a high probability.
Only for degenerated instances, where even in good solutions the target is exceeded by factors $\geq 2$ higher values are possible for $g_1$ and $g_2$.
The value of $g_3$ is guaranteed to be $\leq 1$ because the maximum number of inversions in a permutation of length $k$ is $k\cdot (k-1)/2$.

The final objective function is then a weighted sum of the three normalized objective functions, where the weight $a_i$ of an objective can be seen approximately as its relative importance.
\begin{equation}
\mathbf{minimize}\; g = a_1 \cdot g_1\; +\; a_2 \cdot g_2\; +\; a_3 \cdot g_3
\label{eq:objective_function}
\end{equation}

\subsection*{Quadratic objective function}
For some instances of the \gls{plp} the above presented objective function based on absolute differences may not be well suited.
Intuitively, one could argue that a solution containing $i$ periods with a missing demand of 1 is better compared to an otherwise equal one containing 1 period with missing demand $i$.
However, the above presented objective penalizes the two scenarios in exactly the same way.
In order to penalize larger deviations more than small ones, we introduce an alternative variant of the objective function which calculates the penalty by taking squared differences.
This implies that also the normalization factors need to be adapted.

\begin{align}
\tilde{g}_1 = & \frac{1}{n \cdot (d^*)^2} \cdot \sum_{i \in N} (d^* - w_i)^2	\label{obj1_squ}\\
\tilde{g}_2 = & \frac{1}{n \cdot m}  \cdot \sum_{t \in M} \bigg(\frac{1}{(d_t^*)^2} \cdot \sum_{i \in N} (d^*_t - w_{i,t})^2\bigg) \label{obj2_squ}\\
\tilde{g}_3 = & \, g_3 \label{obj3_squ}
\end{align}
The final objective function using squared differences is again a weighted sum of the three normalized objective functions:
\begin{equation}
\mathbf{minimize}\; \tilde{g} = a_1 \cdot \tilde{g}_1\; +\; a_2 \cdot \tilde{g}_2\; +\; a_3 \cdot \tilde{g}_3
\label{eq:objective_function_squared}
\end{equation}

To our experience the two variants of the objective function behave very similarly for the vast majority of our instances.
Only in some rare cases, where no well-balanced solution is possible and there are large trade-offs to be made, we found that using the quadratic objective function produces results which intuitively look better than the ones produced by the absolute objective.
However, there are also disadvantages to consider when using the alternative objective:
\begin{itemize}
	\item It renders the otherwise linear problem a quadratic one, which makes it harder to solve using \gls{mip}.
	\item Due to the squares in the denominators of the normalization factors, delta costs are often very small which increases the risk of numerical instabilities when using exact solvers.
\end{itemize}

There is also evidence in the literature that there is no a priori reason to prefer one of the two objectives.
\cite{schaus_deviation_2007} investigated balancing objectives in a more general form and stated that a set of violation measures for the perfect balance is given by the $L_p$-norm of a vector of variables $X$ minus its mean $m$ for $p \geq 0$.
For $p = 1$ that corresponds to $\sum_{X \in \mathcal{X}} |X - m|$, which is the same as our linear objective $f_1$.
Similarly, the variant with $p = 2$ corresponds to the quadratic-difference based objective of the \gls{plp}.
As a conclusion of the study of the different variants the authors state that neither criterion subsumes the others~\cite{schaus_deviation_2007}, which means that there is no reason to commit oneself to only one.

\subsection{Related Work}
The term \textit{production leveling} is commonly associated with the \gls{tps}, where it is also called Heijunka.
It is a concept which aims to increase efficiency and flexibility of mass-production by leveling the production in order to keep the stock size low and reduce waste.
Ideally the result of applying Heijunka is zero fluctuation at the final assembly line.
Heijunka can mean both the leveling of volume at the final assembly line and the leveling of the production of intermediary materials~\cite{ohno_toyota_1998}.

The \gls{plp} is clearly inspired by Heijunka in the sense that the usage of resources should be leveled in order to increase production efficiency but its concepts differ quite substantially from the classical implementation of Heijunka (in the \gls{tps}) in the following points:
\begin{itemize}
	\item The \gls{plp} does not operate on the level of schedules but disregards the ordering which the items are produced within a period.
	In other words, it is concerned with planning and not scheduling, which is performed subsequently for each production period.
	\item Intermediate materials are not part of the \gls{plp}.
	While Heijunka aims to level also \textit{their} production to keep stock sizes of intermediary products small, the \gls{plp} is currently only concerned with one level.
\end{itemize}

There exists a whole research area concerning scheduling problems inspired by ideas from the \gls{tps} and especially Heijunka.
Under the umbrella term level scheduling there exist several problems such as the Output Variation Problem and the Product Rate Variation Problem~\cite{kubiak_minimizing_1993, boysen_product_2009}.
They have in common that they aim to find the best schedule for production at the final assembly line so that the demand for intermediary materials and their production is leveled which keeps the necessary stock sizes low.
However, these problems are quite different from the \gls{plp} due to the same reasons presented above with respect to Heijunka.

Under the term \textit{Balancing Problems} several other problems are known in the literature, which are more closely related to the \gls{plp}:
\begin{itemize}
	\item The Balanced Academic Curriculum Problem (\gls{bacp}): This problem deals with assigning courses to semesters such that the student's load is balanced and prerequisites are fulfilled~\cite{chiarandini_balanced_2012}.
	The balancing of the sum of course sizes assigned to a semester is equivalent to the balancing of production load which we are confronted with in the \gls{plp}.
	There exists also variant called the Generalized \gls{bacp} which introduces so-called \textit{Curricula} where each of them should be leveled~\cite{di_gaspero_hybrid_2008}.
	The concept is similar to the product-types of the \gls{plp} except for that an order has only one product type while a course can be in multiple curricula.
	
	The big difference to our problem are additional constraints of the \gls{bacp}, which enforce prerequisites between courses -- a concept appearing frequently in the diverse balancing problems.
	They differ from the \gls{plp}'s priorities in the following aspects:
	Prerequisites are hard constraints while priorities are soft constraints, which makes a difference especially for exact solvers.
	Furthermore, prerequisites require one course to be finished strictly before another starts while it does not seem sensible for the \gls{plp} to require a penalty in case two differently prioritized orders are scheduled to the same period.
	There we only want to penalize when the ordering implied by the priorities is inverted, hence the term \textit{priority inversion}.
	This is a fundamental difference which makes it impossible to solve one problem by converting it to the other.
	
	\item Nurse scheduling problems are an active field of research since their introduction in the 70s~\cite{warner_scheduling_1976}.
	While most of the contributions do not consider workload balancing, a few of them, starting with \cite{mullinax_assigning_2002}, do consider also a fair distribution of the nurses' workload.
	They propose an Integer Programming model for the Nurse to Patient Assignment Problem in neonatal intensive care, which is concerned with finding the optimal assignment of patients to a set of working nurses, so that the workload of the team is balanced and a number of restrictions are fulfilled.
	The main difficulty is the variability of the infant's conditions which greatly influences the amount of work needed.
	The problem is often solved in two steps by first assigning nurses to zones of the nursery and then assigns infants to nurses.
	More recent work in this area is for example by a paper by Schaus et al.\, who investigated a \gls{cp} approach using the \texttt{spread} constraint for balancing~\cite{schaus_scalable_2009}.
	Furthermore, stochastic programming based approaches with Bender's decomposition have been proposed~\cite{punnakitikashem_stochastic_2013}.
	
	The balancing objective of the Nurse to Patient Assignment Problem is again very similar to the objective function which we introduced for the \gls{plp}.
	However, we cannot directly compare to the results because the priorities of the \gls{plp} are have no equivalent in this problem and also vice-versa some side-constraints and the zone assignment cannot be expressed.
	
	\item \gls{salb}: An assembly line consists of identical work stations aligned along a conveyor belt.
	Workpieces move along the conveyor belt and at each station a set of (assembly) tasks is carried out, where each of them has a task time.
	By the cycle time we denote the time after which workpieces are moved on to the next station.
	The goal is either to minimize the number of work stations needed given a fixed cycle time or to minimize the cycle time given a fixed number of work stations.
	
	The \gls{salb} problem is the simplest and most intensively studied variant of Assembly Line Balancing.
	A comprehensive overview over the different variants is provided by~\cite{boysen_classification_2007}.
	When comparing the \gls{salb} problem to the \gls{plp}, tasks map to orders, task times to order sizes and the fixed cycle time to the maximum capacity per production period.
	Hence, minimizing the cycle time is equivalent to minimizing the maximum load of a production period of the \gls{plp}, which would also be an admissible balancing objective.
	There is also recent work by \cite{azizoglu_workload_2018} where the sum of squared deviations of the workstation loads is minimized, which is equivalent to the second variant of the objective which we proposed.
	However, the difference between precedence relations on the one hand and priority inversion minimization on the other hand, disallows once again a direct comparison between the problems.
\end{itemize}

For a more extensive list of Balancing Problems please refer to the dissertation of Pierre Schaus which investigates \gls{cp} modeling approaches for a very diverse set of Balancing and Bin-Packing Problems~\cite{schaus_solving_2009}.

\section{Complexity analysis}
\label{sec:complexity}

As we are studying a new problem we are interested in its computational complexity.
In this section we provide an \NP-completeness proof of a decision variant of the \gls{plp}, followed by an argumentation of \NP-hardness of the \gls{plp} optimization problem presented previously.

In order to prove \NP-completeness, we consider the following decision variant of the problem where the objective function is dropped completely.
Hence the task is solely to find a feasible assignment of orders to periods:

\cprob{Production leveling}
{A set of orders $K$, of products $M$ and of periods $N$.
	For each order $j \in K$ its demand $d_j$ with $d_j > 0$, priority $p_j$ and product type $t_j$.
	The maximum production capacity per period $c$ and for each  product type $t \in M$ its associated maximum production capacity per period $c_t$.}
{Does there exist an assignment $\left\lbrace y_j : N \; | \; j \in K \right\rbrace$ of orders to periods such that the capacity limit $c$ and the capacity limit for each product type $\left\lbrace c_t: t \in M \right\rbrace$ are not exceeded for any period? 
}

\begin{theorem}
	The Production Leveling decision problem is \NP-complete even on instances with $\vert M \vert = 1$, i.e., a single product type.
\end{theorem}

\begin{proof}
	In order to prove \NP-hardness we give a polynomial time reduction from  the \NP-complete Bin Packing decision problem \cite{vazirani_approximation_2003}, which is defined as follows:
	
	\cprob{Bin Packing}
	{A set of $n$ bins $S_1, S_2, \ldots, S_n$ of size $V$ and a list of $k$ items of respective sizes $a_1, a_2, \ldots, a_k$}
	{Can the items be packed into the bins? \\
		I.e., is there an $n$-partition $S_1 \cup S_2 \cup \ldots \cup S_n$ of the set $\left\lbrace 1, 2, \ldots, k\right\rbrace$ such that $\sum_{i \in S_j} a_i \leq V$ for all $j \in \left\lbrace 1, \ldots, n\right\rbrace$?}
	The construction of the \gls{plp} instance is straightforward:
	\begin{compactitem}
		\item $M = \left\lbrace 1 \right\rbrace$
		\item $N = \left\lbrace 1, 2, \ldots, n \right\rbrace$
		\item $K = \left\lbrace 1, 2, \ldots, k \right\rbrace$
		\item $d_j = a_j \quad\forall j \in K$
		\item $p_j = 1 \quad\forall j \in K$
		\item $t_j = 1 \quad\forall j \in K$
		\item $c  = c_1 = V$
	\end{compactitem}
	That is, bins are converted to periods, each item with size $a_i$ to an order with demand $d_i$ and the bin capacity $V$ becomes the maximum capacity per period $c$.
	There is only one product type and order priorities can be defined as some arbitrary constant.
	
	If there exists a feasible solution to this instance of the Production Leveling decision problem (i.e.\ an assignment of orders to periods such that the capacity limit is obeyed), it follows that there exists also a valid bin packing into $n$ bins because each bin with size $V$ corresponds exactly to a period with the same capacity.
	Analogously, if no feasible solution of the \gls{plp} exists we know that the corresponding instance of Bin Packing is infeasible as well.
	Hence any instance of Bin Packing can be solved by converting it to an instance of the Production Leveling decision problem and solving that one.
	As the conversion is possible in linear time, the Production Leveling decision problem must be at least as hard as Bin Packing.
	Consequently, we have proven that it is \NP-hard, even when considering only a single product type ($\vert M \vert = 1$).
	
	In order to prove \NP-membership, let us consider an assignment $\left\lbrace y_j : N \; | \; j \in K \right\rbrace$ of orders to periods. In order to verify whether this assignment is a valid solution, we need to check whether all capacity constraints are fulfilled:
	\begin{compactitem}
		\item Is the overall capacity limit satisfied for each period?
		\[\sum_{\substack{j \in K: \\ y_j = i}} d_j 
		\stackrel{?}{\leq} c \quad \forall i \in N \]
		\item Are the capacity bounds per period and product type satisfied?
		\[\sum_{\substack{j \in K: \\ y_j = i \land t_j = t}} d_j
		\stackrel{?}{\leq} c_t
		\quad \forall i \in N, t \in M \]
	\end{compactitem}
	In total, the number of inequalities that need to be checked is: $\vert N \vert + \vert N \vert \cdot \vert M \vert$ which is clearly polynomial in the size of the instance. 
	
	As the Production Leveling decision problem with only one product type is both \NP-hard and in \NP it is \NP-complete.
\end{proof}

The Production Leveling optimization problem presented in Section~\ref{sec:mathematical-formulation} differs from the decision variant in that we do not only search a feasible assignment of orders to periods but the best possible one according to an objective function.
Obviously, the optimization problem is \NP-hard as well because it needs to satisfy the exact same set of hard constraints.
\NP-membership is, however, not the case as there is no polynomial-time algorithm for deciding whether a given solution is the optimal one.

\section{Integer programming model for the PLP}
\label{sec:MIP}

The \gls{plp} is a weakly constrained optimization problem as the only existing constraints are upper bounds on the planned production volume per period and product.
There are no constraints involved in the prioritization and the objective function is a trade-off which means for example that pruning solutions with a bad priority objective is not immediately possible as long as there is enough room for improvement in the balancing objectives.
Hence the feasible solution space is very large.
We propose an integer programming model for the \gls{plp} which is capable of providing exact solutions to the optimization problem for moderately sized instances.

The model is based on the mathematical formulation presented in Section~\ref{sec:mathematical-formulation}.
The problem input parameters are exactly the same, which is why they are not repeated in this section.
However, because of performance reasons and syntactical restrictions we operate on more and in a few places also different variables.
We introduce a binary view onto the order-period assignment through the variables $X$ and the replacement of the helper variables $w_i$ and $w_{i,t}$ by two variables representing missing and surplus demand.

\subsection*{Variables}
\small
\begin{tabular}{p{2cm}l}
	$x_{ij} \in \{0, 1\}$		& for each $i \in N$, $j \in K$ stating if order $j$ is planned in period $i$\\
	$y_j \in N$    				& for each $j \in K$, whose value is the assigned period of order $j$\\
	$z_{ij} \in \{0, 1\}$     	& for orders $i,j \in K$ where $p_i > p_j$, existence of a priority inversion\\                 & between $i$ and $j$\\
	$s_i^+ \in \mathbb{R}^+$    & for each $i \in N$ the surplus demand for period $i$\\
	$s_i^- \in \mathbb{R}^+$    & for each $i \in N$ the missing demand for period $i$\\
	$s_{it}^+ \in \mathbb{R}^+$ & for each $i \in N$, $t \in M$ the surplus demand for period $i$ and product $t$\\ 
	$s_{it}^- \in \mathbb{R}^+$ & for each $i \in N$, $t \in M$ the missing demand for period $i$ and product $t$
\end{tabular}
\normalsize

\subsection*{Formulation}

\small
\input{methods/mip-model}
\normalsize

Constraints \eqref{eq:mip_one_x} to \eqref{eq:mip_slack_p_constr} are the model's required helper constraints.
Constraint~\eqref{eq:mip_one_x} makes sure, that there is exactly one period to which an order is assigned.
Constraint~\eqref{eq:mip_link_xy} links the $x_{ij}$ to the $y_i$ variables.
Constraint~\eqref{eq:mip_link_yz} links the $y_i$ to the $z_{i,j}$ variables.
It makes sure that for every pair of orders $i,j$ where $i$ has a higher priority than $j$, $z_{ij}$ is 1 (representing an inversion) if $i$ is planned later than $j$.
Constraint~\eqref{eq:mip_slack_constr} states for each period that the total demand planned plus the surplus minus the slack equals $d^*$.
As both variables have positive domains and they are subject to minimization, at most one of them will be non-zero in any optimal solution.
Constraint~\eqref{eq:mip_slack_p_constr} repeats this relationship over the variables $s_{it}^+$ and $s_{it}^-$ for each product type $t$.

Constraint~\eqref{eq:mip_cmax_total} ensures that the capacity bound per period is satisfied.
This is elegantly achieved by stating that the sum of target demand $d^*$ and the surplus variable $s^+$ does not exceed the threshold.
Analogously, Constraint~\eqref{eq:mip_cmax[p]} enforces the capacity limit per period and product type.

Finally, there are two redundant constraints for strengthening the formulation:
Constraint~\eqref{eq:mip_symmetry} enforces a dominance relation for all pairs of orders which have the same product type and demand value.
The constraint requires that the higher prioritized order occurs not later than the lower prioritized one which is sensible because otherwise we could swap the two orders to obtain a better solution.
This cuts off parts of the search space where the optimal solution cannot reside.
Constraint~\eqref{eq:mip_link_slacks} links the $s_i^{\{+,-\}}$ and $s_{it}^{\{+,-\}}$ variables together, which also leads to improvements in the average runtime.
	
\subsection{Absolute-difference-based objective}

The following objective function is equivalent to the one presented in Section~\ref{sec:mathematical-formulation} but here it is stated on the variable set of the MIP formulation.
It is not hard to see that in function $g_1$ the sum of the slack and surplus variable $(s_i^+ + s_i^-)$ is equivalent to the absolute difference between planned an target demand $| d^* - w_i |$, because at least one of $s_i^+$ and $s_i^-$ will be 0 in any optimal solution and the other one holds the absolute difference.
The same holds true for the analogous variables in $g_2$.

\begin{align}
g_1 = & \frac{1}{n \cdot d^*} \cdot \sum_{i \in N} (s_i^+ + s_i^-)\\
g_2 = & \frac{1}{n \cdot m}  \cdot \sum_{t \in M} \bigg(\frac{1}{d_t^*} \cdot \sum_{i \in N} (s_{it}^+ + s_{it}^-)\bigg)\\
g_3 = & \frac{2}{k \cdot (k-1)} \cdot \sum_{i,j \in K} z_{i,j}\
\end{align}

\subsection{Squared-difference-based objective}

The second variant of the objective function which uses squared differences can also be easily expressed:

\begin{align}
\tilde{g}_1 = & \frac{1}{n \cdot (d^*)^2} \cdot \sum_{i \in N} (s_i^+)^2 + (s_i^-)^2\\
\tilde{g}_2 = & \frac{1}{n \cdot m}  \cdot \sum_{t \in M} \bigg(\frac{1}{(d_t^*)^2} \cdot \sum_{i \in N} (s_{it}^+)^2 + (s_{it}^-)^2\bigg)\\
\tilde{g}_3 = & g_3
\end{align}

The main difference to the first version is that the surplus and missing demand appears squared in the objective function.
Furthermore, the normalization factors have been adapted.

Obviously, this formulation is no longer a linear program.
However, as many state-of-the art solvers also support quadratic optimization we consider this variant worth reporting.

\section{Local search for the PLP}
\label{sec:local_search}

As shown earlier, the \gls{plp} is an \NP-hard optimization problem, i.e.\ it belongs to a class of problems for which no polynomial-time algorithms have been found so far and it is even unclear whether such algorithms exist.
Consequently, there exist instances for which the required running time of any known algorithm to find the exact solution is exponential.
It means also that exact solution approaches are impractical for solving large instances of the \gls{plp} and we need to take heuristic methods into account.

In this section we present metaheuristic local search techniques to solve the \gls{plp}.
To obtain initial solutions we present two different approaches.
Afterwards two neighborhood structures for the \gls{plp} are described and finally we explain the local search algorithms.

\subsection{Construction of initial solutions}
\label{sec:greedy}

We developed a greedy construction heuristic which is capable of constructing good initial solutions in a very small amount of time.
The parameters of the algorithm are a list of orders, the number of periods $n$ and the random selection size $r$.
The first step of the algorithm is sorting the orders by priority decreasingly which is already the approximate handling of objective 3.
Then we loop over all periods $i$ from 1 to $n$, performing the following steps:
\begin{enumerate}
	\item Examine sequentially the orders from the head of the sorted order list:
	For each of them, if it still fits into this period obeying the capacity limits, calculate the delta cost for $g_1$ and $g_2$ (as defined in \eqref{obj1_abs} and \eqref{obj2_abs} in Section~\ref{sec:mathematical-formulation}) which the inclusion of this order into the period would bring with it.
	If the delta cost is smaller than zero (i.e.\ including the order improves the objectives), it is added to a list of suitable orders.
	The orders from the head of the sorted list are processed in this way until the suitable order list has size $\frac{k}{n}$ (i.e.\ the average number of orders for each period) or there are no orders left.
	\item Afterwards, if the list is not empty, select randomly one of the $r$ best suitable orders, plan it for period $i$, remove it from the sorted order list and go back to 1.
	\item Otherwise (if there was no suitable order) repeat with $i := i+1$.
\end{enumerate}
Finally, we check whether there are any orders left which could not be assigned due to the capacity limits.
If that is the case, they get assigned one by one to the period with maximal remaining capacity.
This way especially those periods which are not filled well get assigned the remaining orders and the probability of a hard constraint violation is minimized.
However, violating the maximum capacity constraint is allowed in this step because a complete assignment is required for the subsequent local search.

The parameter $r$ controls the random selection size of step 2.
If we set it to 1, the algorithm is deterministic.
When using values greater than 1, the construction heuristic is randomized, which can be useful for some local search techniques (e.g. GRASP).

\subsection{Neighborhood structures}

We devised two types of moves for generating different neighborhoods of a solution which will be introduced in the following subsections.
Furthermore, we briefly describe the delta evaluation approach and the methods of neighborhood exploration.

\subsubsection{Move-order neighborhood}

The move-order neighborhood (or simply move neighborhood) of a solution $s$ consists of all solutions whose only difference to $s$ is that one order has been moved to a different period.
Figure~\ref{fig:move_neighborhood_example} visualizes such a move.
The figure on the left shows the leveling objective per product type before the move and on the right side we can see the result of applying the move.
Order 2 is moved from $P2$ to $P3$ which yields in this case a better solution.

Enumerating the move neighborhood involves iterating over $k$ orders for each of $n-1$ possible target periods, i.e.\ the neighborhood size is exactly $k \cdot (n-1)$.

\begin{figure}
	\includegraphics[width=\textwidth]{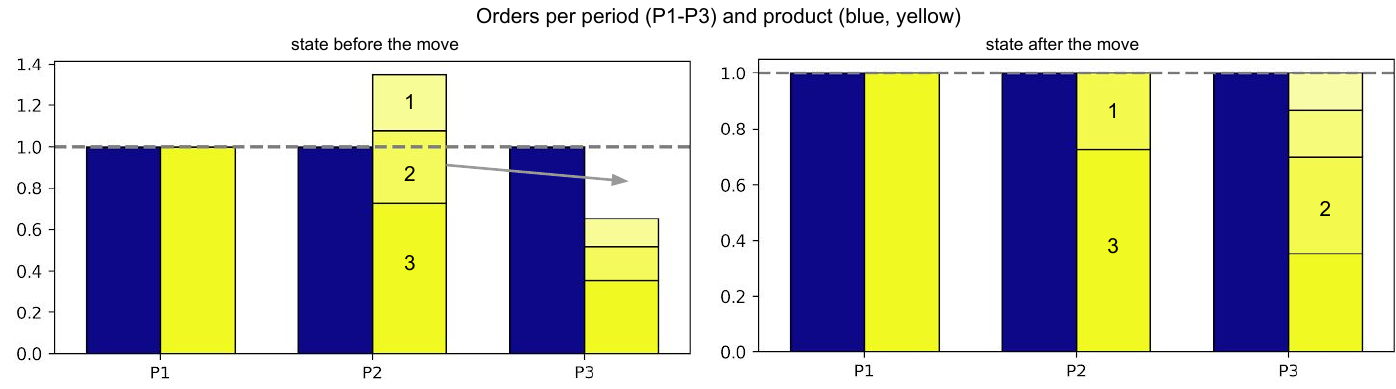}
	\centering
	\caption{Example of a move-order move: solutions before (left) and after (right)}
	\label{fig:move_neighborhood_example}
\end{figure}

\subsubsection{Swap orders neighborhood}

The swap-orders neighborhood (or simply swap neighborhood) of a solution $s$ consists of all solutions $s'$ whose only difference to $s$ is that two orders not assigned to the same period in $s$ appear with swapped period assignments in $s'$.
Figure~\ref{fig:swap_neighborhood_example} visualizes such a move.
Order 1 is swapped with order 2 which in this case again yields a better solution.

Enumerating the swap neighborhood involves iterating over all pairs of orders not assigned to the same period. Hence the neighborhood size is in $O(k^2)$.

\begin{figure}
	\includegraphics[width=\textwidth]{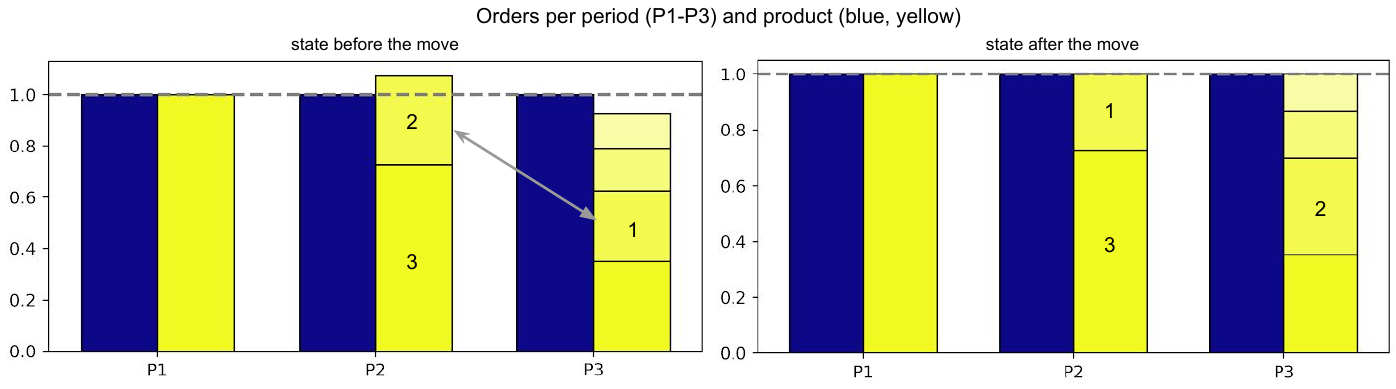}
	\centering
	\caption{Example of a swap-orders move: solutions before (left) and after (right)}
	\label{fig:swap_neighborhood_example}
\end{figure}

\subsubsection{Neighborhood exploration}
\label{sec:neighborhood_traversal}

We investigated three types of neighborhood exploration:
\begin{itemize}
	\item \textbf{First Improvement}: Generate and evaluate moves until the point where the first move is found who would improve the current solution.
	In order to prevent a bias towards the start of our neighborhood (e.g. the first orders in our input) the neighborhood traversal is performed in a cyclic way.
	That is, instead of starting every time at the same point we start right after the position where we found the first improving move the last time and search until either an improving move is found or we arrive again at the point where we started.
	\item \textbf{Best Improvement}: Generate and evaluate the complete neighborhood of a solution and select the move which leads to the biggest improvement.
	Ties are broken randomly.
	\item \textbf{Random Neighbor}: Generate and evaluate a random neighbor of the given solution.
\end{itemize}

\subsubsection{Move evaluation}

In order to explore a neighborhood systematically, we need to be able to compare moves with respect to their quality.
Given two moves $a$ and $b$, the first criterion to check is the number of hard constraint violations which each of them introduces or resolves.
If $a$ introduces fewer or resolves more of them we say that $a$ is better than $b$.
Otherwise -- if the number of hard constraint violations is equal -- we compare by selecting the one which has the lower move cost, which is defined as the change of the current solution's objective value if we would perform this move.

To avoid costly complete evaluations of whole solutions we propose a delta evaluation that efficiently evaluates how much the objective value changes for a given move.
The delta evaluation implementations for the two move types both use the same primitive for evaluating the cost of moving one order to a different period.
When performing swaps, we calculate the cost of moving order one to the period of order two, the cost for moving order two to the period of order one and compensate the error which results from assuming in both calculations that the respective other order remains unchanged.
The delta cost of moving an order is calculated for the three objective function components separately:
\begin{enumerate}
	\item For the leveling objective we only need to keep track of the planned production volume for each period, so that we can calculate the effect the move on the difference to the target value.
	\item For the per-product leveling objective we can do the same thing, given that we keep track of the planned production volume for each period and product.
	\item The priority objective is the hardest and most time-consuming part of delta evaluation because moving an order from period $i$ to $j$ can introduce or resolve inversions between the moved order and every order assigned to a period between $i$ and $j$.
	When the number of orders is very large it is inefficient to iterate over all such orders and perform comparisons because we need to do that for every candidate move.
	Our idea for optimizing this evaluation is based on the insight that the only thing we care about when moving an order past a period is the number of orders in that period which have smaller and larger priorities, respectively, not the actual priority values.
	Therefore, we maintain the priority values of all orders assigned to a certain period in a sorted list (one for each period), so that we can efficiently retrieve via binary search how many orders have smaller / larger priorities than the order which we currently want to move.
\end{enumerate}
The delta cost of the three objective function components is aggregated to a single value by the usual formula for the objective value \eqref{eq:objective_function}.

\subsection{Algorithms}
\label{sec:algorithms}

In this section we present details of the metaheuristic local search methods which we investigated for solving the \gls{plp}, namely the simple and deterministic \gls{vnd} as well as Simulated Annealing.

\subsubsection{Variable Neighborhood Descent}
\label{sec:vnd}

\gls{vnd} is a deterministic local search technique which can be seen as an extension of hill climbing to multiple neighborhoods.
The general idea is to go on to the next neighborhood if the current one gets stuck in a local optimum and return to the first one as soon as a further improvement is found.
The selection of an improving move in the neighborhood is usually done by using a deterministic exploration technique, i.e.\ either first or best improvement.
Algorithm~\ref{alg:vnd} shows the details using pseudo code, as it is given in the Handbook of Metaheuristics~\cite{hansen_variable_2010}.
\begin{algorithm}
	\SetAlgoNlRelativeSize{-1}
	
	\KwData{$initialSolution$, neighborhoods $\mathcal{N}_1, \ldots, \mathcal{N}_k$, $timeLimit$, $iterationLimit$}
	\KwResult{a solution at least as good as $initialSolution$}
	
	$currentSolution \leftarrow initialSolution$\;
	$iterationCount \leftarrow 1$\;
	$j \leftarrow 1$\;
	
	\While{$j \leq k$ \textbf{and}
		$\lnot$ out of time \textbf{and}
		$iterationCount \leq iterationLimit$}{
		$bestMove \leftarrow$ select a neighbor of $currentSolution$ w.r.t.\ $\mathcal{N}_j$\;
		\eIf{$bestMove$ is an improvement}{
			$currentSolution \leftarrow$ doMove($bestMove$)\;
			$j \leftarrow 1$\;
		}
		{
			$j \leftarrow j + 1$\;
		}
		
		$iterationCount \leftarrow iterationCount + 1$\;
	}
	
	\KwRet $currentSolution$\;
	
	\caption[Variable Neighborhood Descent]{Variable Neighborhood Descent}
	\label{alg:vnd}
\end{algorithm}

The idea of using multiple neighborhoods is based on the following insights~\cite{hansen_variable_2010}:
\begin{itemize}
	\item A local optimum w.r.t.\ one neighborhood structure is not necessarily a local optimum w.r.t.\ another.
	\item A global optimum is a local optimum w.r.t.\ all possible neighborhood structures.
\end{itemize}
That implies it is beneficial to use several complementary neighborhoods and try to escape local optima of one neighborhood by switching to another.

\subsubsection{Simulated Annealing}
\label{sec:simulated-annealing}

Simulated Annealing is a metaheuristic optimization method introduced by \cite{kirkpatrick_optimization_1983}.
It resembles the physical process of annealing in metallurgy insofar as both methods use a cooling schedule in order to control the amount of random movements in the process, which in theory allows for convergence to the optimal state.
Even though convergence to the optimal solution is usually not achieved in practical settings, Simulated Annealing is still one of the most widely used metaheuristic optimization methods.

Given an initial solution, a set of neighborhoods $\mathcal{N}_i$ with associated probabilities $p_i$, the starting temperature $t_{max}$, minimum temperature $t_{min}$, number of iterations per temperature $w$, time limit and iteration limit the version of Simulated Annealing we propose works as shown in Algorithm~\ref{alg:sa}.
\begin{algorithm}
	
	\SetAlgoNlRelativeSize{-1}
	
	\KwData{$initialSolution$, neighbohoods $\mathcal{N}_i$ with probabilities $p_i$, $t_{max}$, $t_{min}$, iterations per temperature $w$, $timeLimit$, $iterationLimit$}
	\KwResult{a solution at least as good as $initialSolution$}
	
	$currentSolution \leftarrow initialSolution$\;
	$bestSolution \leftarrow currentSolution$\;
	$t \leftarrow t_{max}$\;
	
	\While{$t \geq t_{min}$ \textbf{and} $\lnot$ time limit reached \textbf{and} $\lnot$ iteration limit reached}{
		\ForEach{$j \in 1, \ldots, w$}{
			$\mathcal{N} \leftarrow$ choose one of neighborhoods $\mathcal{N}_i$ according to probabilities $p_i$\;
			$m \leftarrow$ select a random move out of $\mathcal{N}(currentSolution)$\;
			\If{Accept($m$, $t$)}{
				$currentSolution \leftarrow$ Apply($m, currentSolution$)\;
				\If{$currentSolution$ is better than $bestSolution$}{
					$bestSolution \leftarrow currentSolution$\;
				}
			}  
		}
		$t \leftarrow$ Cool-Down($t$)\;
	}
	
	\KwRet $bestSolution$\;
	
	\caption[Simulated Annealing for the \gls{plp}]{Simulated Annealing}
	\label{alg:sa}
\end{algorithm}

The pseudo code makes use of two functions \textbf{Accept}, standing for the acceptance criterion, and \textbf{Cool-Off}, defining the cooling schedule, which we discuss in the following:
\begin{itemize}
	
	\item \textbf{Acceptance Criterion:}
	We use the metropolis criterion as acceptance function, which was introduced in the original paper by \cite{kirkpatrick_optimization_1983}.
	The probability of acceptance $P(i \Rightarrow j)$ of a move from solution $i$ to solution $j$ (for the case of minimization), with $f(x)$ standing for the objective value of solution $x$, can be defined as follows:
	\begin{equation}
	P(i \Rightarrow j) = \begin{cases}
	1, 										& \text{if } f(j) \leq f(i).\\
	exp\Big(\frac{f(i) - f(j)}{t}\Big), 	& \text{otherwise}.
	\end{cases}
	\label{eq:metropolis}
	\end{equation}
	If the candidate solution $j$ is at least as good as the current solution $i$, it is accepted unconditionally.
	Otherwise it is accepted with a probability which is decreasing exponentially as a function of the negative delta cost divided by the current temperature.
	That means, if a candidate solution is much worse than the current one it will be accepted with a lower probability than a solution which is just a little bit worse.
	
	\item \textbf{Cooling schedule:}
	The temperature is decreased during the search process by means of a cooling schedule which is usually a geometric row.
	In our case it depends on the cooling rate $\alpha$ and the iterations per temperature level $w$.
	The function \textbf{Cool-Down()} reduces the temperature after every $w$ iterations by the following formula:
	\begin{equation}
	t_i = \alpha \cdot t_{i-1}
	\end{equation}
	
	We want to stress now briefly how $\alpha$ and $w$ interact.
	Assuming we are given an iteration limit $l$, the initial temperature $t_{max}$ and the final temperature $t_{min}$ there exist many different options to reach $t_{min}$ after $l$ iterations, namely all combinations of $\alpha$ and $w$ such that $t_{min} = \alpha^n\cdot t_{max}$ where the number of temperature steps $n = \left\lfloor\frac{l}{w}\right\rfloor$.
	Two examples of schedules following that formula with $l = 30000$, $t_{max} = 1$ and $t_{min} = 0.001$ are depicted in Figure~\ref{fig:cooling-schedules}.
	Please observe that for both options depicted in the figure the temperature at each time is approximately the same as the different step sizes and widths compensate each other.
	Therefore, it is sufficient to fix the cooling rate $\alpha$ when tuning the parameters of Simulated Annealing and let the cooling schedule be determined only by the variation of $w$.
	
	\begin{figure}
		\centering
		\includegraphics[width=0.5\textwidth]{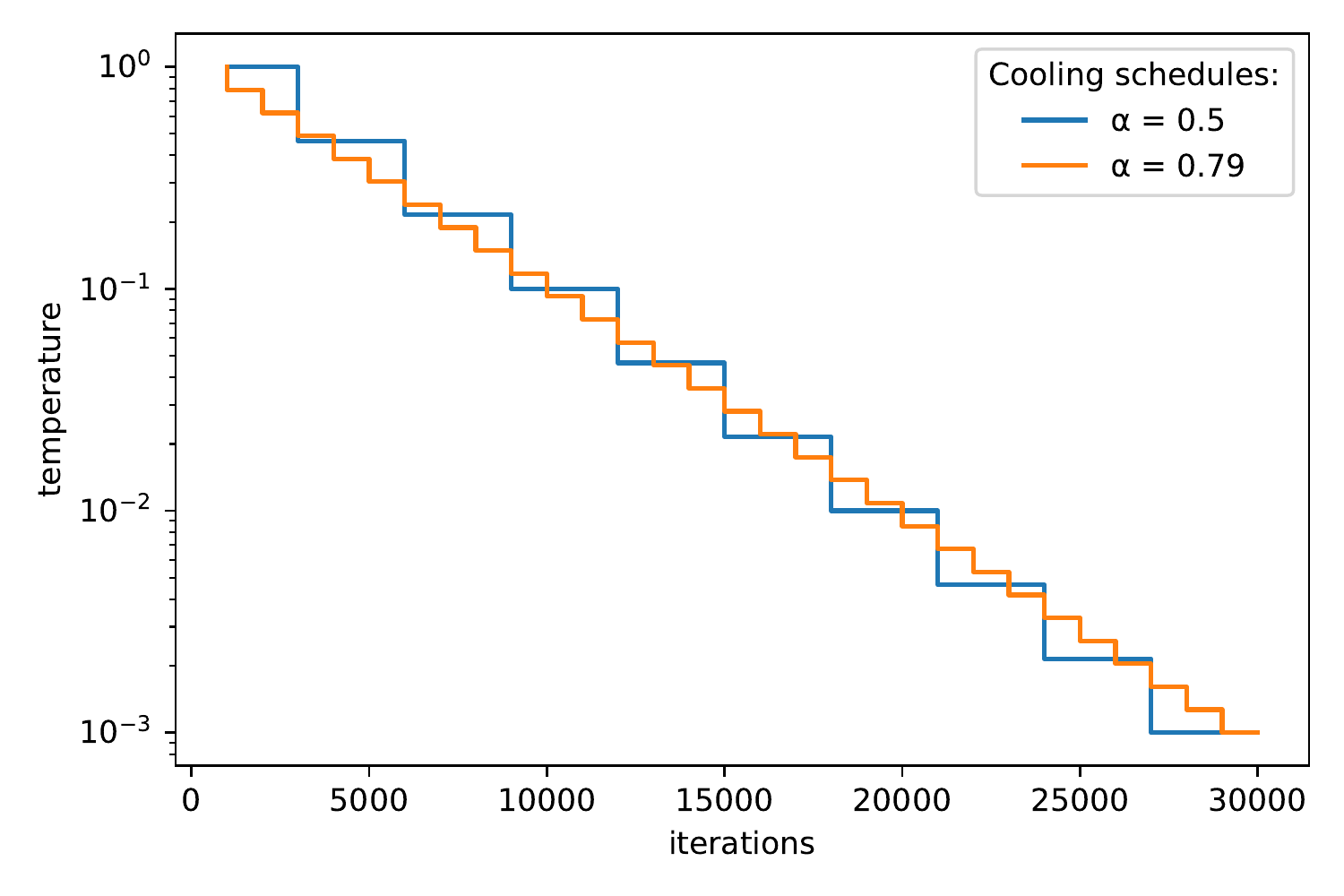}
		\caption[Two cooling schedules with different cooling rates]{Two cooling schedules with different cooling rates and iterations per temperature but identical start and end points}
		\label{fig:cooling-schedules}
	\end{figure}
	
	If we do not know the number $l$, we can also derive a formula which relates two cooling schedules $(\alpha_1, w_1)$ and $(\alpha_2, w_2)$ that have the same slope:
	\begin{equation}
	\frac{w_1}{w_2} = \frac{\log \alpha_1}{\log \alpha_2} \label{eq:cooling-rates}
	\end{equation}
	Using this relationship one can construct alternative cooling schedules which decrease equally fast on average.
	
\end{itemize}

\section{Experimental Evaluation}
\label{sec:evaluation}

In this section we evaluate the practical contributions of our work and provide answers to the questions which have been raised.
As the \gls{plp} is a new problem we initially elaborate on the problem instances and propose two instance generation procedures.
Next we describe properties of the test set, define parameters and describe the processing environment.
After that we turn towards the actual evaluation and look at the \gls{mip} model in detail.
Ultimately the metaheuristic approaches are extensively evaluated.

\subsection{Problem Instances}

The problem we describe emerges from a real-life use case of our industrial partner which also provided us some data from the production system.
In total we received 27 \gls{plp} instances which all have 20 periods, 4 to 8 product types and 79 to 1585 orders.
This set of instances will be called from now on $R_1$.
As these instances do not suffice for a thorough evaluation and we do not want to restrict ourselves to their size, we designed also two random instance generation procedures which are described in the following.

\subsubsection{Perfectly solvable instances}
We devised a method of generating instances which allow for a perfectly balanced solution with zero cost, that we know from the construction process.
That is, of course, a restriction of generality, but it is extremely useful as a means of evaluating the optimality gap for large instances which would otherwise be impossible as we have currently no way of solving them exactly with usual compute resources.
Despite the existence of a perfectly balanced solution with no priority inversions the instances are still not easy to solve to optimality, at least not as long as you don't provide the information of perfect realizability to the solvers.

The instance generation process relies on the subroutine for random integer partitioning shown in Algorithm~\ref{alg:integer_partitioning}.
It takes as arguments the integer to partition, the number of partitions and a minimum value for each partition.
The main idea is to represent the number $n$ as an array of $n - k \cdot minV$ zeros and then inserting $k-1$ ones at random positions.
In the resulting array an integer partition of the number $n - k \cdot minV$ into $k$ parts can be found by looking at the number of zeros between every two neighboring ones.
Finally we add $minV$ to every element of the result array to obtain the requested partition with minimum value.

\begin{algorithm}
	\caption{Integer partitioning algorithm}
	\label{alg:integer_partitioning}
	
	\KwData{\textit{n}, \textit{k}, \textit{minV}}
	\KwResult{An array with k integers whose sum equals n, each of which being $\geq minV$}
	
	\textbf{let} \textit{array} $\leftarrow$ an array consisting of $n - (k \cdot \textit{minV})$ zeros\;
	Insert $k-1$ ones into \textit{array} at random positions\;
	\textbf{let} \textit{spaces} $\leftarrow$ number of zeros between the ones in \textit{array}\;
	add \textit{minV} to every element of \textit{spaces}\;
	\textbf{return} \textit{spaces}\;
\end{algorithm}

Using this partitioning algorithm, Algorithm~\ref{alg:instance_generation} defines the procedure for generating random instances with a fixed number of orders, periods and products.
First the total number of orders is partitioned into one part for each period where each part has to have at least as many orders as we have products.
This is important because every product needs to meet its target in every period in order to achieve an objective value of 0.
The same thing is done for each period to decide upon the number of orders for each product type.

Next we draw the overall target value for the production volume (which is the same for each period) by taking the desired \textit{avgDemandPerOrder} and multiplying with the average number of orders per period plus a random deviation of at most $10\%$.
Then we partition that value into one part for each product, which is the demand for each product per period.

Finally we need to partition the demand for each product, which we decided upon in line 4, into the number of orders for each period and product which we calculated in line 2.
The priorities must be chosen such that no inversion can exist, which is achieved by assigning each period a range of priority values decreasingly such that the ranges do not overlap, and choosing for each order randomly one of the allowed values.
From this data the order and product type list can be built, which completes the instance.
The optimal solution is known as well from the construction process.

\begin{algorithm}
	\caption{Procedure for the creation of perfectly solvable instance}
	\label{alg:instance_generation}
	
	\KwData{\textit{m}, \textit{n}, \textit{k}, \textit{avgDemandPerOrder}}
	\KwResult{A realizable instance with $m$ products, $n$ periods, $k$ orders and the optimal solution}
	
	\textbf{let} \textit{ordersPerPeriod} $\leftarrow$ partition($k, n, m$)\;
	\textbf{let} \textit{ordersPerPeriodAndProduct} $\leftarrow$ partition(\textit{ordersPerPeriod}[$o$], $m$, 1) for every order $o$\;
	\textbf{let} \textit{plannedDemand} $\leftarrow$ $\frac{k \cdot \mathit{avgDemandPerOrder}}{n} \pm 10\%$\;
	\textbf{let} \textit{plannedDemandPerProduct} $\leftarrow$ partition(\textit{plannedDemand}, $m$, max(\textit{ordersPerPeriodAndProduct}))\;
	\textbf{let} \textit{orderDemands} $\leftarrow$ partition(\textit{plannedDemandPerProduct}[$p$], \textit{ordersPerPeriodAndProduct}[$t,p$], 1) for every period $t$ and product $p$\;
	\textbf{let} \textit{allowedPriorities} $\leftarrow$ for each period $n$ a distinct set of priorities s.t. they decrease with increasing $n$\;
	\textbf{let} \textit{orderPriorities} $\leftarrow$ choose for each order one of the priorities which are allowed according to the period of the order\;
	build the list of orders and products and shuffle them\;
	assign random product names\;
	\textbf{return} a new solution from the list of orders and products and the optimal solution\;
\end{algorithm}

Using Algorithm~\ref{alg:instance_generation} we generated 1000 instances, sampling the parameters for each one independently as follows:
The number of orders $k$ is chosen from $100\ldots4000$, the number of periods $n$ from $2\ldots80$, the number of products $m$ from $1\ldots20$ and \textit{avgDemandPerOrder} from $5\ldots500$.
The resulting set of instances is subsequently called $R_2$.



\subsubsection{Random instances}

We also devised a second instance generation procedure where the optimal solutions are not known by design and we can't even guarantee that there exists a feasible one, which is surely a more practice-oriented approach.
The instances are designed to share some properties of the 27 realistic instances:
\begin{itemize}
	\item There exist only a limited number $l \ll k$ of different order demand values. This means we frequently see repeated orders which may have different priorities though.
	\item Orders of different products draw their demand data from different distributions. Whereas product $a$ may have demand values between $0$ and $1000$, product $b$ may have it between $0$ and $5000$.
	\item Sometimes there exist product types whose number of orders is smaller than the number of periods which implies that the demand for some periods will exceed the target while for others it must be zero.
\end{itemize}

The actual generation process is very simple. Given a number of orders $k$, periods $n$ and product types $m$ the algorithm works as follows:
\begin{enumerate}
	\item Partition the number of orders $k$ into $m$ parts $c_1 \ldots c_m$.
	\item Choose the maximum priority of all orders $p_{max} \in [1; 3n]$
	\item Choose $1-50$ allowed demand values $d \in [1; random(1000-5000)]$ for each product $p$, named $D_p$.
	\item For each product $p \in [1; m]$, generate $c_p$ orders, choosing the demand from the set $D_p$ and the priority from $[0; p_{max}[$.
\end{enumerate}

Using this procedure, we generated the instance set $R_3$ consisting of 1000 instances by sampling the parameters randomly as it has been done above with the other procedure.
The number of orders $k$ is chosen from $100\ldots4000$, the number of periods $n$ from $2\ldots80$ and the number of products $m$ from $1\ldots20$.
Furthermore, we generated a set of 10 small instances, named $R_4$, where the number of orders $k$ is chosen from $30\ldots100$, the number of periods $n$ from $5\ldots20$ and the number of products $m$ from $1\ldots5$.

\subsection{Experimental Setting}

The instances which are described above are split into training and test set so that the parameter tuning is not executed on the same instances as the validation.
The test set consists of the whole set of realistic instances $R_1$, 50 instances of $R_2$, 50 instances of $R_3$ and all 10 instances in $R_4$.
Table~\ref{tab:instance-sets} provides an overview of the instance sets and the way they were split.

\begin{table}[]
	\centering
	\caption{Overview over the different instance sets and the split into training and test set}
	\label{tab:instance-sets}
    \footnotesize
	\begin{tabular}{@{}llrllll@{}}
		\toprule
		& Name                                & Count & Description                           & Training Set Selection & Test Set Selection &  \\ \midrule
		$R_1$ & \texttt{realistic\_instance}        & 27    & Realistic instances                   & -                      & 01-27              &  \\
		$R_2$ & \texttt{randomly\_perfect}          & 1000  & Randomly generated perfectly solvable & 0001-0950              & 0951-1000          &  \\
		$R_3$ & \texttt{randomly\_generated}        & 1000  & Randomly generated                    & 0001-0950              & 0951-1000          &  \\
		$R_4$ & \texttt{randomly\_generated\_small} & 10    & Randomly generated, small             & -                      & 1-10               &  \\ \bottomrule
	\end{tabular}
\end{table}

We chose to build the test set out of four different instance types because we wanted to make sure that our algorithms can cope with different characteristics and sizes.
The size distribution is shown by Table~\ref{tab:test_set} which states for each instance parameter --- $k$ (number of orders), $m$ (number of product types), and $n$ (number of periods) --- the minimum, maximum and mean value on each part of the test set.
The smallest instances are $R_4$, followed by the realistic instances $R_1$.
The instances coming from $R_2$ and $R_3$ are much larger on average as we want to evaluate also the scalability of our algorithms.
\begin{table}[]
	\centering
	\footnotesize
	\caption{Minimum, maximum, mean and standard deviation of number of orders $k$, number of product types $m$ and number of periods $n$ for every part of the test set}
	\label{tab:test_set}
	\begin{tabular}{llrrrr}
		\toprule
		&       &  min &   max &    mean &     std \\
		\textbf{Parameter} & \textbf{Instance Set} &      &       &         &         \\
		\midrule
		\multirow{4}{*}{\textbf{k}} & \textbf{$R_1$} &   79 &  1585 &  307.19 &  412.56 \\
		& \textbf{$R_2$} &  105 &  3896 & 1595.86 &  954.09 \\
		& \textbf{$R_3$} &  112 &  3991 & 2076.76 & 1207.02 \\
		& \textbf{$R_4$} &   34 &    98 &   61.20 &   19.70 \\
		\cline{1-6}
		\multirow{4}{*}{\textbf{m}} & \textbf{$R_1$} &    4 &     8 &    6.93 &    1.24 \\
		& \textbf{$R_2$} &    1 &    19 &    8.82 &    5.50 \\
		& \textbf{$R_3$} &    1 &    19 &    9.04 &    5.48 \\
		& \textbf{$R_4$} &    1 &     4 &    2.80 &    1.03 \\
		\cline{1-6}
		\multirow{4}{*}{\textbf{n}} & \textbf{$R_1$} &   20 &    20 &   20.00 &    0.00 \\
		& \textbf{$R_2$} &    4 &    78 &   39.50 &   22.48 \\
		& \textbf{$R_3$} &    4 &    77 &   39.26 &   22.04 \\
		& \textbf{$R_4$} &    7 &    18 &   10.90 &    4.04 \\
		\bottomrule
	\end{tabular}
\end{table}
The set of test instances is publicly available on the following web-page:
\href{https://dbai.tuwien.ac.at/staff/jvass/production-leveling}{https://dbai.tuwien.ac.at/staff/jvass/production-leveling}.

We use the absolute-difference-based objective function~\eqref{eq:objective_function} to produce all the subsequent results.
The reason is that in our setting the advantages over the objective with squares outweigh the disadvantages, especially because it enables us to solve much more instances exactly.
The evaluation of the metaheuristics could just as well be done using the quadratic objective function~\eqref{eq:objective_function_squared} but as we want to compare to the exact results we use formula \eqref{eq:objective_function} as well.
Hard constraint violations are not part of the objective function but undergo a special treatment where possible by reporting the number of violated constraints as a separate number or separate plot.
In some cases, e.g. statistical significance tests, we handle objective and constraint violations at once by adding them up.
Due to the small magnitude of the objective a penalization factor for hard constraint violations is not necessary.

As indicated during the problem statement, we worked out default values for the weights of the objective function components $a_1, a_2$ and $a_3$ in cooperation with our industrial partner, namely $1$, $1$ and $\frac{1}{3}$, respectively.
All experiments of the evaluation are using this weighting.

Wherever nothing different is stated the algorithm parameters are defined as follows:
\begin{itemize}
	\item The greedy heuristic has only one parameter, $r$, for controlling the amount of randomness, which we set to 1 (i.e.\ deterministic) for all the experiments.
	The parameter is only necessary for some local search techniques like GRASP which would require a randomized construction heuristic.
	\item For \gls{vnd} the move neighborhood is used first because it can be enumerated very quickly.
	Only when no improving move can be found any more the larger swap neighborhood gets employed.
	This ordering leads to a much quicker termination because the first neighborhood is searched much more often than the second.
	As the neighborhood exploration strategy we use Next Improvement with restart at the last position (as defined in Section~\ref{sec:neighborhood_traversal}), which showed at least equal performance to Best Improvement in preliminary experiments.
	\item The parameters of Simulated Annealing are tuned automatically.
	The concrete process and the results get introduced later on.
	\item The \gls{mip} model is executed using Gurobi Optimizer 8.1.1~\cite{gurobi} on a single thread and otherwise the default settings.
\end{itemize}

All experiments were conducted on a computing cluster with with 10 identical nodes, each having 24 cores, an Intel(R) Xeon(R) CPU E5-2650 v4 @ 2.20GHz and 252 GB of memory, running Ubuntu 16.04.1 LTS.
The metaheuristic algorithms are implemented in C\# and executed using Mono 4.2.1.

\subsection{Evaluation of the MIP model}
\label{sec:mip-eval}

In this subsection we will examine the \gls{mip} model presented in Section~\ref{sec:MIP} with respect to its empirical performance.
We first break down the results by the different instance sets of which the test set is composed.
Afterwards we will investigate how the instance size affects the solution quality.

First, we want to investigate how well each part of the test set can be solved using \gls{mip}.
For a description of the different parts please refer to Table~\ref{tab:test_set}.
Figure~\ref{fig:mip_status_grouped_by_set} visualizes the shares of optimally solved, feasibly but not optimally solved, infeasible and unsolved instances per group $R_1$ to $R_4$.
\begin{figure}
	\centering
	\includegraphics[width=.5\linewidth]{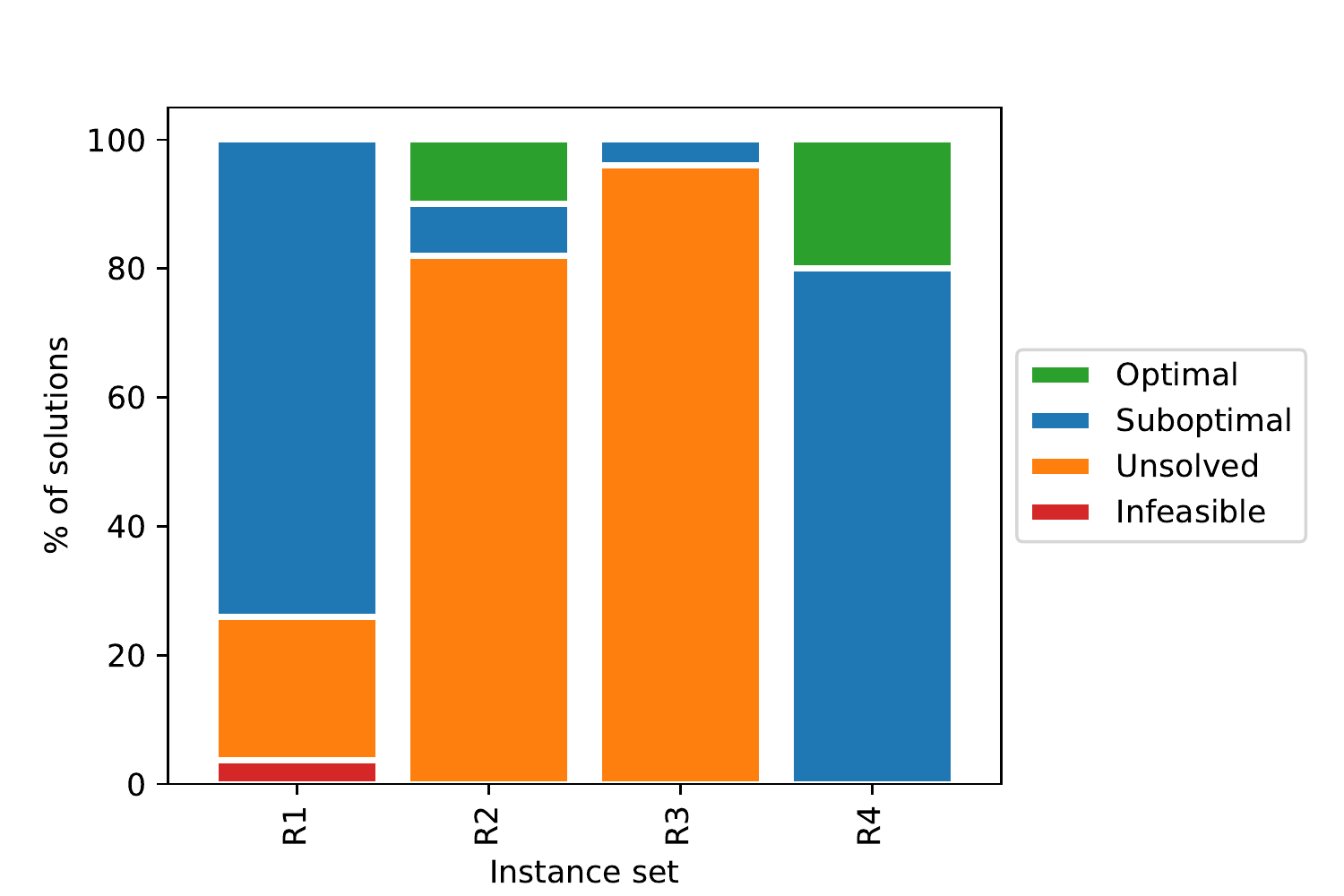}
    \caption{Share of solution statuses of \gls{mip} for each subset of the test set. Optimal means proven optimal. Suboptimal implies that an integer solution has been found but it has not been proven that it is optimal. Unsolved means that within the time limit no integer solution has been found and it is thus unclear whether there exists a feasible solution at all. Infeasible means that the solver proved that no feasible solution exists.}
    \label{fig:mip_status_grouped_by_set}
\end{figure}
The most noticeable difference between the sets is that in $R_2$ and $R_3$ the vast majority of the instances are unsolved while for the other two most of the instances are solved (but still not proven optimal).
Presumably, the reason for that is that most of the instances in these sets are very large.
An interesting fact is, though, that about 10 \% of the instances in $R_2$ could be solved to proven optimality but not a single one in $R_3$ even though the instance sizes of the two sets have been sampled from the same distribution.
One potential reason for that could be that the instances in $R_2$ are designed to have optimal solutions with objective value 0.
That should make optimality proofs easy for the solver once the optimal solution has been found because no part of the objective function can by negative.

For the instance sets $R_1$ and $R_4$ over $70 \%$ of the instances end up with some solution, which is not proven optimal.
We want to investigate how good these solutions are and use for that purpose the relative optimality gap with respect to the best lower bound.
It is calculated as the percentage corresponding to one minus the ratio between bound cost and incumbent cost.
Table~\ref{tab:mip-optimality-gap} shows the minimum, maximum and mean optimality gap as well as the standard deviation for all suboptimal solutions in $R_1$ and $R_4$.
With an average gap of only $3.96\%$ the realistic instance set $R_1$ is solved really well, so that the \gls{mip} model might be usable in practice when the instances are not too large and a runtime of one hour is not an issue.
On the other hand, $R_4$ has a low minimum and a high maximum gap as well as a large standard deviation.
That means that the randomly generated instances are quite difficult to solve using \gls{mip}, even though the ones in $R_4$ are mostly smaller than the realistic ones.

\begin{table}[]
	\centering
	\caption{Optimality gap of MIP for suboptimal instances in $R_1$ and $R_4$}
	\label{tab:mip-optimality-gap}
	\begin{tabular}{lrrrr}
		\toprule
		{} &        min &        max &       mean &       std \\
		\midrule
		\textbf{$R_1$} &  $0.99\%$ &  $11.65\%$ &  $3.96\%$ &  $2.55\%$ \\
		\textbf{$R_4$} &  $0.63\%$ &  $98.85\%$ & $31.14\%$ & $41.52\%$ \\
		\bottomrule
	\end{tabular}
\end{table}

Finally, we investigate in more detail how the instance size correlates with the results of the \gls{mip} model.
Figure~\ref{fig:mip_status_grouped_by_kmn} visualizes the solution statuses of all instances in the test set, grouped by the number of orders $k$, the number of products $m$ and the number of periods $n$, from left to right.
The most apparent relationship is a correlation between the number of orders $k$ and the percentage of unsolved instances.
While below 250 orders almost every instance has either been proven feasible or infeasible, the share of unsolved solutions increases drastically when increasing $k$.
When looking at the middle and right-hand-side figure, we can see that the share of unsolved instances is also increasing with increasing number of periods and product types but it starts already quite high in the smallest bin.
We can conclude that instances with 250 orders or less can be solved with a high probability by the \gls{mip} model, but there is no such bound which we could state on the number of periods or product types.
While increasing $n$ and $m$ clearly complicates the problem, making them small does not automatically make the problem easy to solve.

\begin{figure}
	\includegraphics[width=\textwidth]{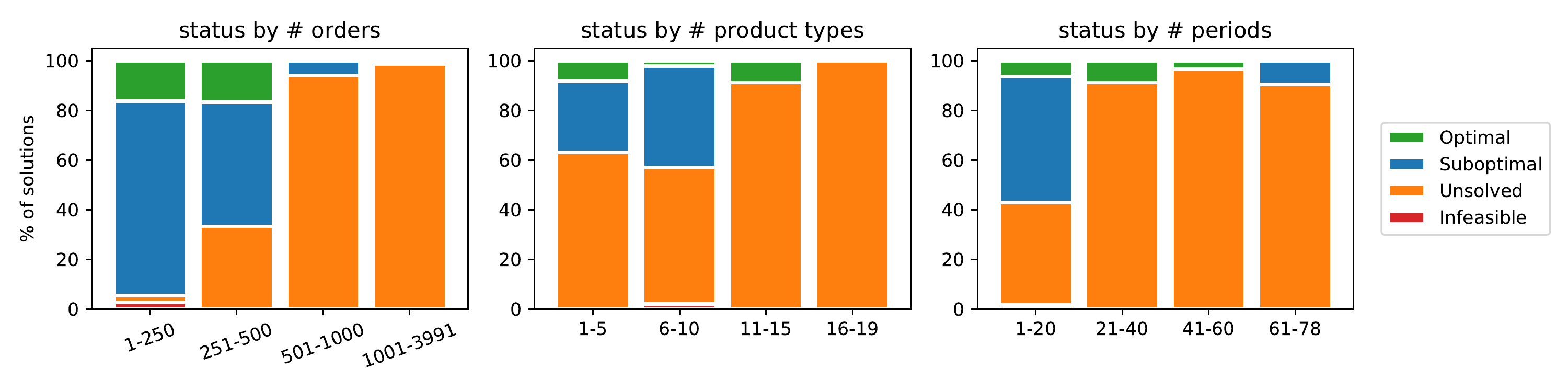}
	\caption{Solution statuses of \gls{mip} on the test set, grouped by value ranges of instance features. Optimal means proven optimal. Suboptimal implies that an integer solution has been found but it has not been proven that it is optimal. Unsolved means that within the time limit no integer solution has been found and it is thus unclear whether there exists a feasible solution at all. Infeasible means that the solver proved that no feasible solution exists.}
	\label{fig:mip_status_grouped_by_kmn}
\end{figure}

\subsection{Evaluation of metaheuristics}

As we have seen, the \gls{mip} formulation is not suitable for solving large instances, which is why we developed local search methods for the \gls{plp} as well.
In this subsection we will first deal with the automatic parameter tuning for Simulated Annealing and validate the claim that it is sound to fix the cooling rate.
Then we analyze benefits and shortcomings of our two approaches \gls{vnd} and Simulated Annealing in detail on the basis of results on our test set, comparing also against the greedy heuristic.
Thereafter we examine the sensitivity of Simulated Annealing to variations in the weighting of the neighborhoods.
Finally we examine how close the metaheuristic solutions get to the global optima by using dual bounds obtained through \gls{mip} and the perfectly solvable instance set $R_2$.

\subsubsection{Algorithm Configuration}

As described in Section~\ref{sec:algorithms}, Simulated Annealing depends on parameters whose setting has a huge influence on the algorithm's efficiency and effectiveness.
We deal with their configuration by means of \gls{smac}, an automatic algorithm configuration tool written in python.
It relies on Bayesian Optimization in combination with an aggressive racing mechanism in order to efficiently search through huge configuration spaces~\cite{lindauer_smac_2019}.

We applied \gls{smac} to tune the parameters of Simulated Annealing as it was presented above.
The set of instances which was use for the tuning can be found in the column \textit{Training Set Selection} of Table~\ref{tab:instance-sets}.
The parameter optimization was executed for 24 hours on 24 cores in parallel.
We used a time limit of five minutes per run and no iteration limit.
The cooling rate was not tuned but set to a value of $0.95$, which is not a restriction of generality as long as the number of iterations per temperature can still be adjusted (see Figure~\ref{fig:cooling-schedules}).
This claim will be verified in a separate experiment later on.

We tuned the initial temperature $t_{max}$, the number of iterations per temperature $w$ and the probability $p$ that the move neighborhood is used to generate the next random move (hence $1-p$ is the probability of the swap neighborhood).
Tuning the minimum temperature $t_{min}$ is not necessary because the results cannot get worse when Simulated Annealing is run until the time limit instead of aborting when the minimal temperature is reached.
Indeed, preliminary results showed that setting the minimum temperature to zero instead of using the tuning results of \gls{smac} improves results to a small but significant extent.
The configuration space with minimum and maximum values as well as the defaults and the tuning result is shown in Table~\ref{tab:config-space-sa}.

\begin{table}[]
	\centering
	\caption{Configuration space of Simulated Annealing}
	\label{tab:config-space-sa}
	\begin{tabular}{@{}llrrrr@{}}
		\toprule
		\textbf{Parameter} & \textbf{Type} & \textbf{Minimum} & \textbf{Maximum} & \multicolumn{1}{l}{\textbf{Default}} & \multicolumn{1}{l}{\textbf{Tuned}} \\ \midrule
		Iterations Per Temperature & integer & $10^3$ & $10^6$ & $10^3$ & $2.52 \cdot 10^{5}$\\
		Move Neighborhood Probability (\%) & integer & 0 & 100 & 50 & 40 \\
		Initial Temperature & real & 0.1  & 10.0 & 5.0  & 0.22 \\
		\midrule
		Minimum Temperature & real (fixed) & 0    & 0    & 0    & 0 \\
		Cooling Rate        & real (fixed)& 0.95 & 0.95 & 0.95 & 0.95  \\ \bottomrule
	\end{tabular}
\end{table}

\subsubsection{Experiments about fixing the cooling rate}

We claimed in Section~\ref{sec:simulated-annealing} that the cooling rate $\alpha$ could be set to a constant value because it was redundant as long as the number of iterations per temperature $w$ is free.
During algorithm configuration we did exactly that and set $\alpha \leftarrow 0.95$.
Now we want to verify this claim by means of an experiment.
We derive four more cooling schedules from the one defined by the result of parameter tuning whose temperature profile follows the same slope.
Then we benchmark each configuration on the whole test set ten times with different random seeds and take the median of the objective values and number of constraint violations for each instance.

In Section~\ref{sec:simulated-annealing} we already introduced an equation which allows to derive cooling schedules with equal average slopes.
We selected the alternative cooling rates $0.5$, $0.75$, $0.9$ and $0.99$ and computed the associated values of $w$.
A summary of the resulting cooling schedules is shown in Table~\ref{tab:cooling-schedules}.

\begin{table}[]
	\centering
	\caption{Five equivalent cooling schedules which have the same slope on average. The value for $w$ in the line with $\alpha = 0.95$ comes from parameter tuning and the rest has been derived so that the slope is does not change.}
	\label{tab:cooling-schedules}
	\begin{tabular}{@{}rr@{}}
		\toprule
		\multicolumn{1}{l}{\textbf{Cooling Rate $\mathbf{\alpha}$}} & \multicolumn{1}{l}{\textbf{Iterations per temperature $\mathbf{w}$}} \\ \midrule
		0.50 & 3412581 \\
		0.75 & 1416349 \\
		0.90 & 518723 \\
		\textit{0.95} & \textit{252533} \\
		0.99 & 49481 \\ \bottomrule
	\end{tabular}
\end{table}

Figure~\ref{fig:experiment-cooling-rate} shows on the left hand side a box plot for each of the schedules, each of them plotting the median objective value resulting from the derived schedule divided by the median objective value resulting from the original schedule, per instance.
The original schedule does, of course, not differ from itself, while the other schedules bring an improvement for some instances and worse results for others.
For the three higher values of $\alpha$ more than 50\% of the data is in the range $\pm 1\%$ and nearly all the rest in $\pm 2\%$ (except for a couple of outliers outside the plotting range).
The median shown by the box plots of the schedules $\alpha=0.9$ and $\alpha=0.99$ is almost exactly at 1 whereas the other two variants have median differences slightly higher than one and the whiskers reach farther, although the differences are still very small.
The right hand side of Figure~\ref{fig:experiment-cooling-rate} visualizes the same for the number of hard constraint violations, except that we do not divide but take the difference between the alternative schedules and the original one because the interpretation is more intuitive in this case.
The whole box plot except for some outliers is zero which means that for most of the instances there is no change in the number of hard constraint violations.
However, there are a few outliers going down to -1 which means that for these instances the alternative schedules have one violation less.
Compared to the 137 instances of the test set the number of outliers is very small though.

\begin{figure}
    \centering
    \begin{subfigure}{.49\textwidth}
        \centering
    	\includegraphics[width=\linewidth]{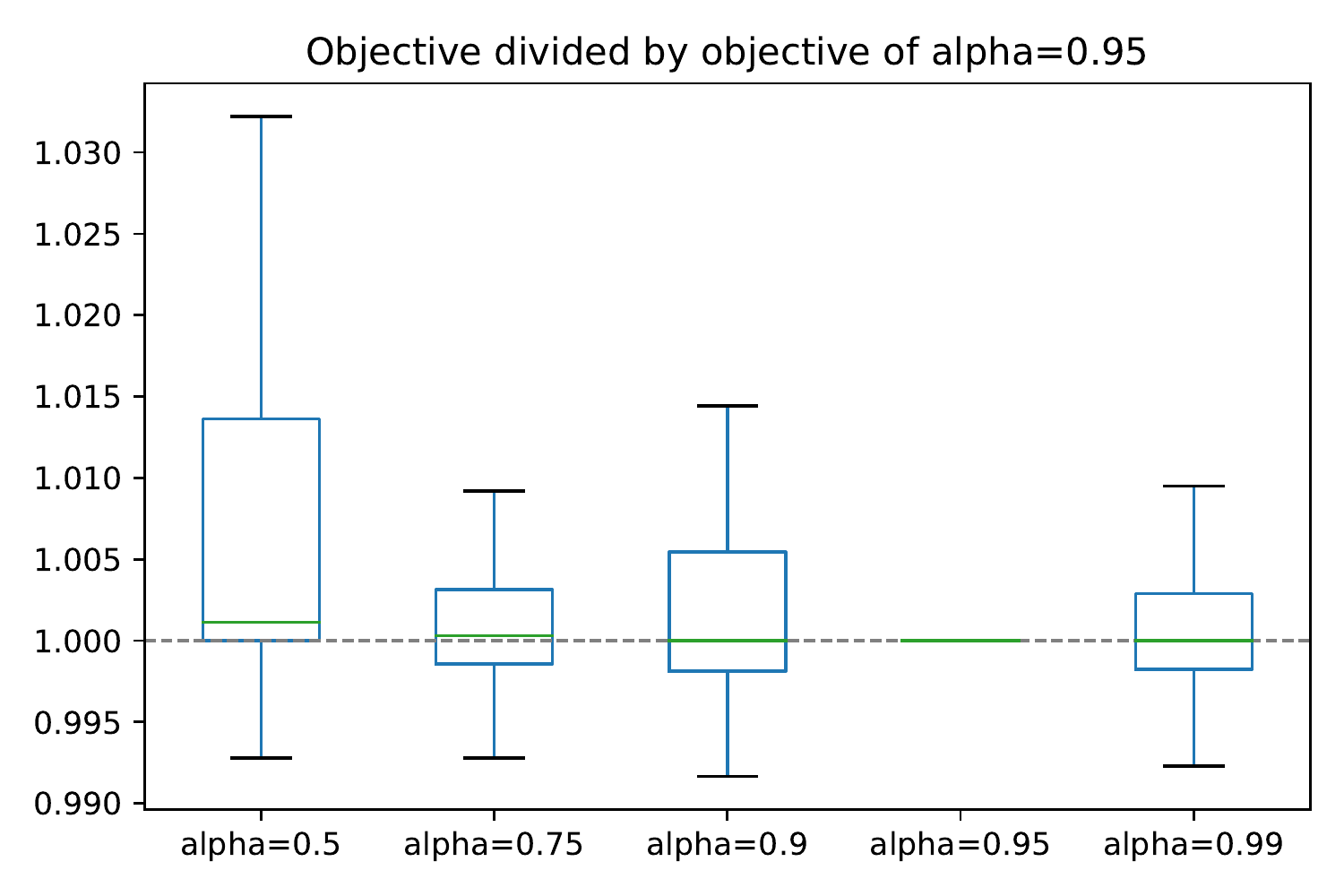}
    \end{subfigure}
    \begin{subfigure}{.49\textwidth}
        \centering
    	\includegraphics[width=\linewidth]{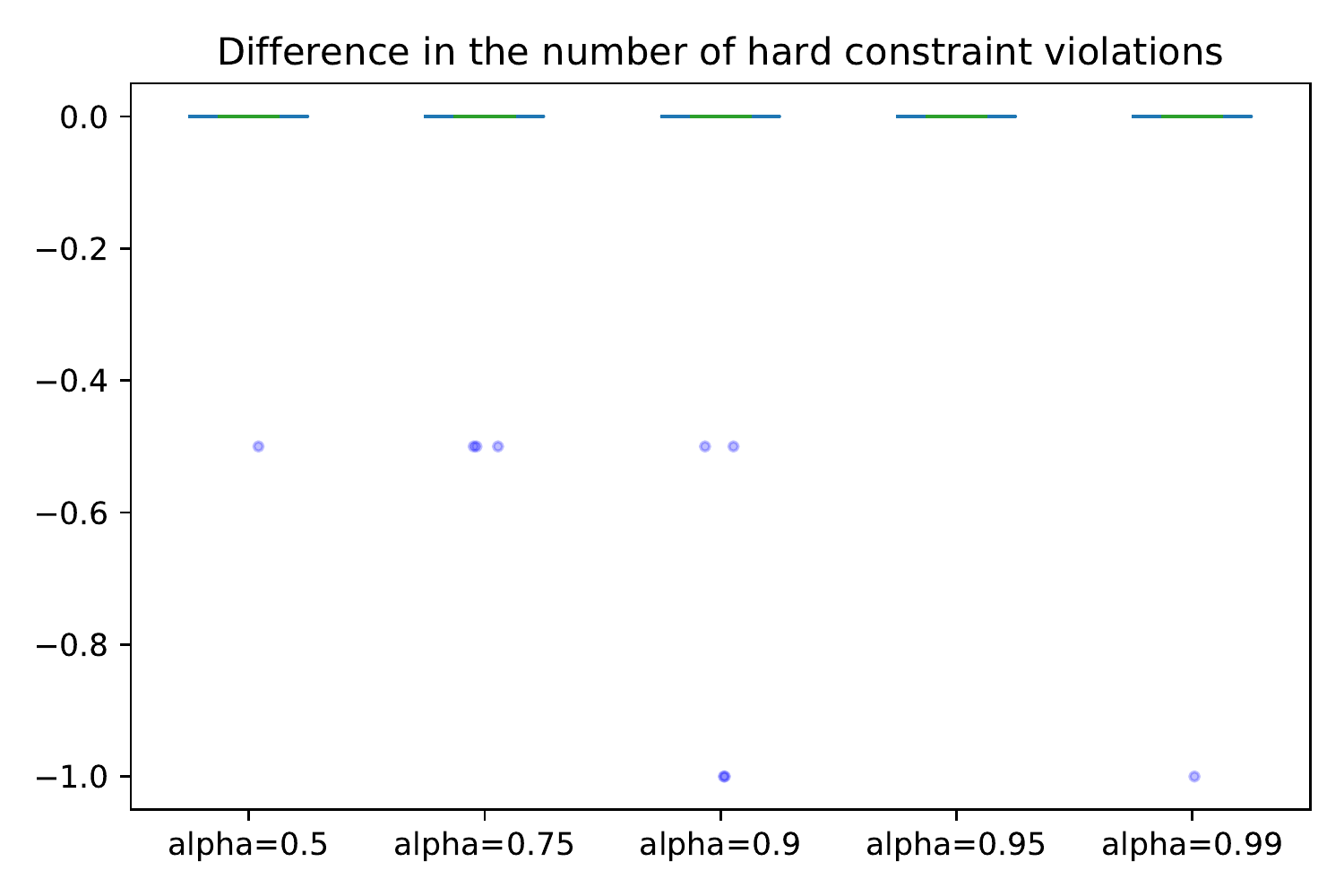}
    \end{subfigure}
	\caption{Results of the experiment regarding different cooling schedules. All the results are comparisons per instance against the schedule with $\alpha = 0.95$.}
	\label{fig:experiment-cooling-rate}
\end{figure}

In order to clarify whether the differences which we see in the median of $\alpha=0.5$ and $\alpha=0.75$ are statistically significant or not we conducted a Wilcoxon signed-rank test between the original schedule and each of the derived ones.
In order to take also the hard constraint violations into account their number was added to the objective value as a penalty.
The null hypothesis was that the median of the differences is zero and the alternative that it is different from zero.
We want to reject the null hypothesis in case we find a p-value which is smaller than $0.05$.
The result of the tests yielded a p-value of $2.6 \cdot 10^{-8}$ for $\alpha = 0.5$, $0.041$ for $\alpha = 0.75$, $0.966$ for $\alpha = 0.9$ and $0.26$ for $\alpha = 0.99$.
That implies that we must reject the null hypothesis for the schedules with $\alpha \leq 0.75$.
For the other two schedules the statistical test does not let us reject the null hypothesis that the differences which we see are the result of chance.

To sum up, the claim that we can change alpha without changing the result is valid in our setting as long as $\alpha$ is set high enough (i.e.\ in our experiment at least 0.9).
That implies that we were on the safe side when fixing it to 0.95 during parameter tuning.
For values $\alpha \leq 0.75$ there exists a tiny but statistically significant difference which corresponds to a median objective value which is about $0.1\%$, above the one of the default configuration.

\subsubsection{Comparison of metaheuristic techniques}

We want to compare now the presented heuristic and metaheuristic techniques for solving the \gls{plp}, namely the greedy algorithm presented in Section~\ref{sec:greedy}, and \gls{vnd} and Simulated Annealing which have been presented in Section~\ref{sec:algorithms}.
Therefore, these three algorithms were benchmarked on the test set with the usual settings.
For Simulated Annealing we conducted 10 runs to account for randomness in the search process and aggregated the runs by taking the median value of each measure.
In order to be able to compare objective values and hard constraint violations visually, we report for each instance the difference to the best solution we ever obtained using any method and time limit.
\footnote{This is a sound approach because the objective function is already normalized so that the instance size does not have an influence on the magnitude of the objective. Using a ratio instead of the difference, like in the previous experiment, is not possible here because the best known solution for the instances in $R_2$ have objective value 0 which yields a division by 0.}

The result is shown by Figure~\ref{fig:metaheuristic-comparison}:
To the left one can see the objective values of the three approaches.
The median difference between the greedy heuristic's objective values and the best known ones is about $0.07$.
Furthermore, we can see in the center figure that a considerable number of instances could not be solved without hard constraint violations by the Greedy heuristic.
Expressed in numbers, that's the case for 63 out of 137 instances or some $46\%$.
Compared to the Greedy, \gls{vnd} delivers solutions with much better objective values and fewer constraint violations, but there are still 33 solutions or $24\%$ where at least one capacity constraint is violated.
Simulated Annealing achieves the best median objective value of all methods and furthermore the fewest instances with constraint violations (22 out of 137, $16\%$).
The non-overlapping notches of the box plot in the left figure indicate that the difference to \gls{vnd} is significant.
The rightmost plot shows the solving time of the different methods in seconds.
The greedy heuristic needs always less than a second of time.
\gls{vnd} is also mostly fast, because the search continues only until a local optimum w.r.t.\ all neighborhood structures is found, which takes long only for the very largest of our instances.
Simulated Annealing uses always the complete available time because we don't stop at a minimum temperature in order to maximize the solution quality.

\begin{figure}
    \centering
    \begin{subfigure}{.33\textwidth}
        \centering
    	\includegraphics[width=\linewidth]{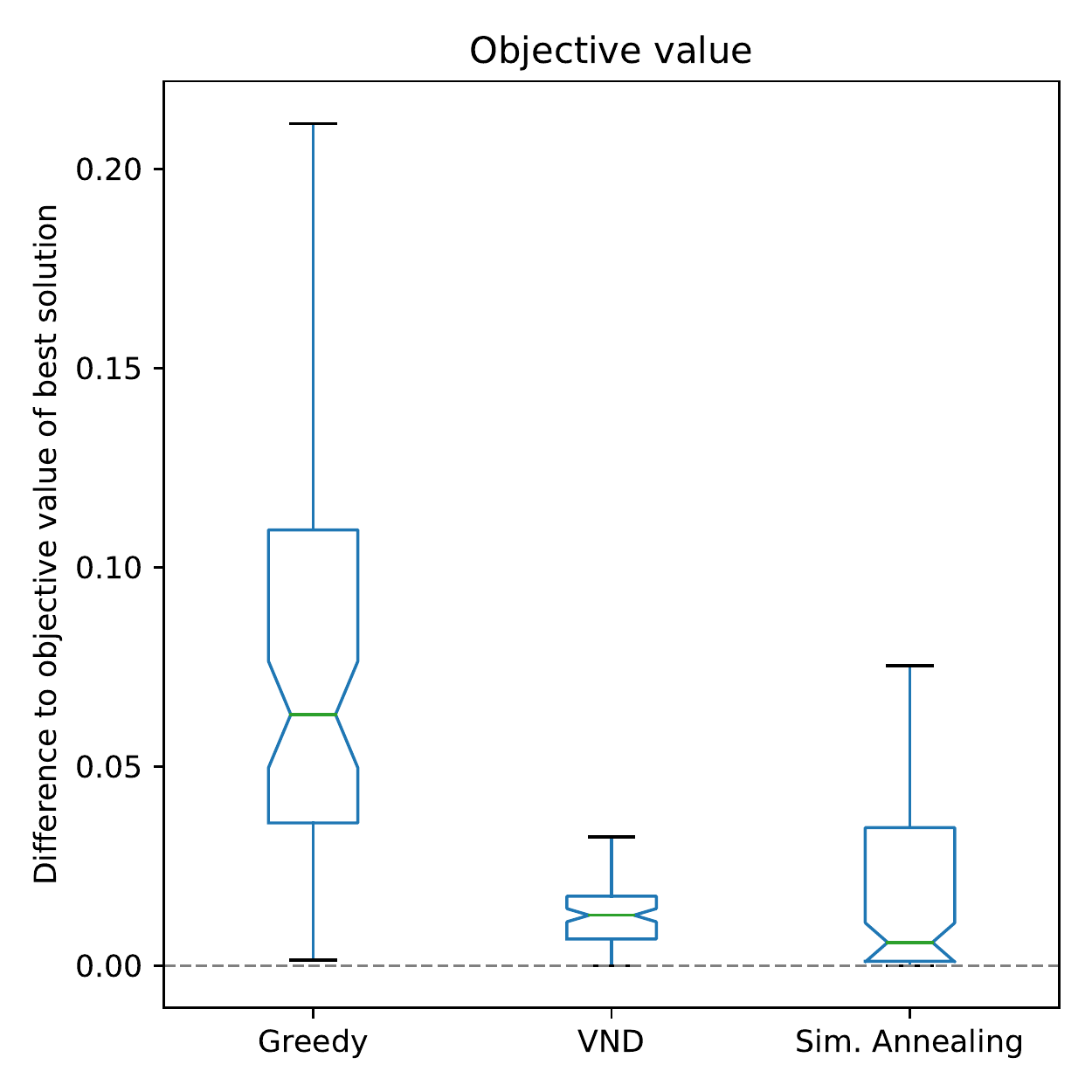}
    \end{subfigure}
    \begin{subfigure}{.33\textwidth}
        \centering
    	\includegraphics[width=\linewidth]{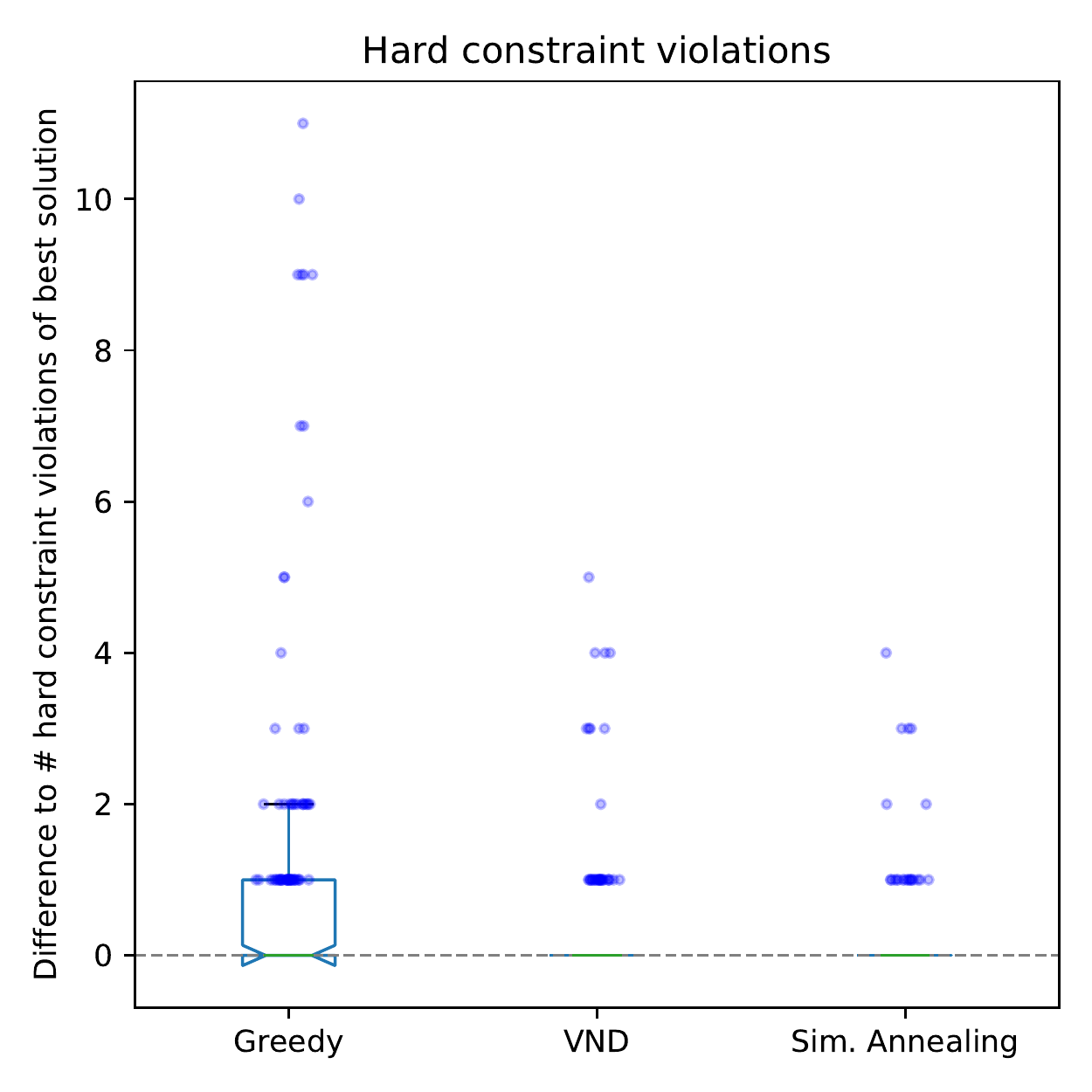}
    \end{subfigure}
    \begin{subfigure}{.33\textwidth}
        \centering
    	\includegraphics[width=\linewidth]{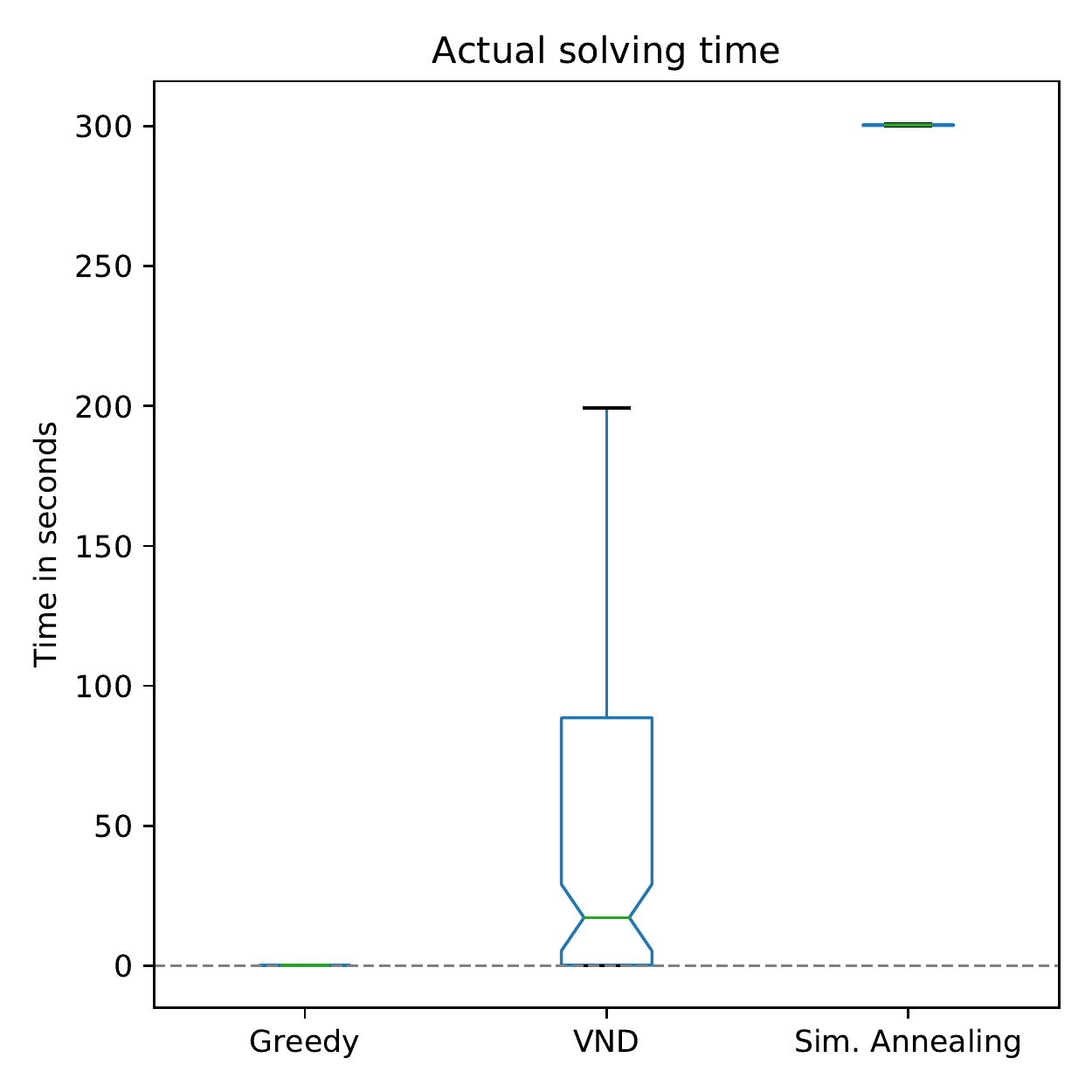}
    \end{subfigure}
    \caption{Comparison of objective values, hard constraint violations and solving time of the Greedy heuristic, \gls{vnd} and Simulated Annealing on the test set.}
    \label{fig:metaheuristic-comparison}


\end{figure}

In the supplementary materials a table with the objective value and the number of hard constraint violations for each instance of the test set for all reported algorithms can be found.

\subsubsection{Sensitivity to neighborhood weightings in Simulated Annealing}

Before selecting the next random move in the Simulated Annealing algorithm the neighborhood is chosen randomly according to some weighting.
This weighting has been tuned by \gls{smac}, resulting in a probability of $0.4$ for the move neighborhood and $0.6$ for the swap neighborhood.
The tuning progress of \gls{smac} revealed quite large fluctuations in this weighting.
Therefore, we conduct a sensitivity analysis in order to find out what impact different weightings have on the results.

We evaluate 6 different weightings, one of which is the result of \gls{smac}.
The probability $p$ for the move neighborhood in the 6 scenarios ranges from 0 to 1 in steps of $0.2$ and the probability of the swap neighborhood is the complementary probability $1 - p$.
Each configuration is executed on the test set 10 times with the usual time limit of five minutes.
The runs are again aggregated using the median.

Figure~\ref{fig:experiment-neighborhood-weighting} shows to the left a box plot for each of the alternative weightings, each of them plotting the associated objective value divided by the objective value of the original weighting.
The labels '$x$ - $y$' mean that the move neighborhood has weight $x$ and the swap neighborhood has weight $y$.
The objective value gets worse in the extreme cases which can be seen because the leftmost and the two rightmost boxes lie completely above the dashed line.
The other cases are practically equal which means that the objective value does not change compared to the reference weighting.

In the right plot of Figure~\ref{fig:experiment-neighborhood-weighting} the difference of the number of hard constraint violations to the respective number of the reference weighting '40 - 60' is shown.
All boxes are completely flat contained in $\{0\}$ which means that the inter-quartile range is equal for all weightings.
The outliers show that a few instances of the extreme configurations have one or two more hard constraint violations more than the reference configuration, while the weightings '20 - 80' and '60 - 40' have tendentially fewer.
However, compared to the number of 137 instances contained in the test set the amount of outliers is very small, so that the significance of this difference must be doubted.

\begin{figure}
    \centering
    
    \begin{subfigure}{.49\textwidth}
        \centering
    	\includegraphics[width=\linewidth]{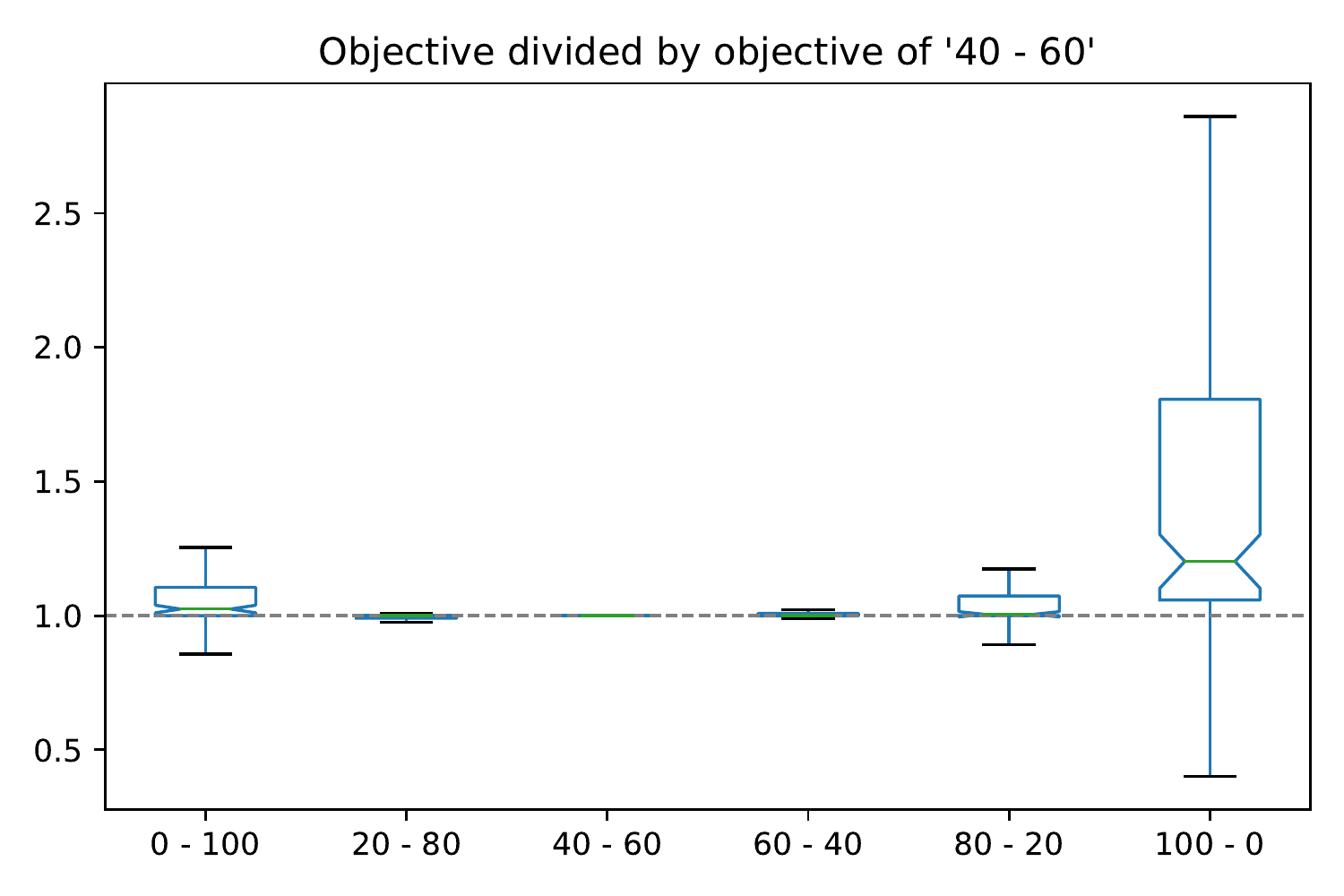}
    \end{subfigure}
    \begin{subfigure}{.49\textwidth}
        \centering
    	\includegraphics[width=\linewidth]{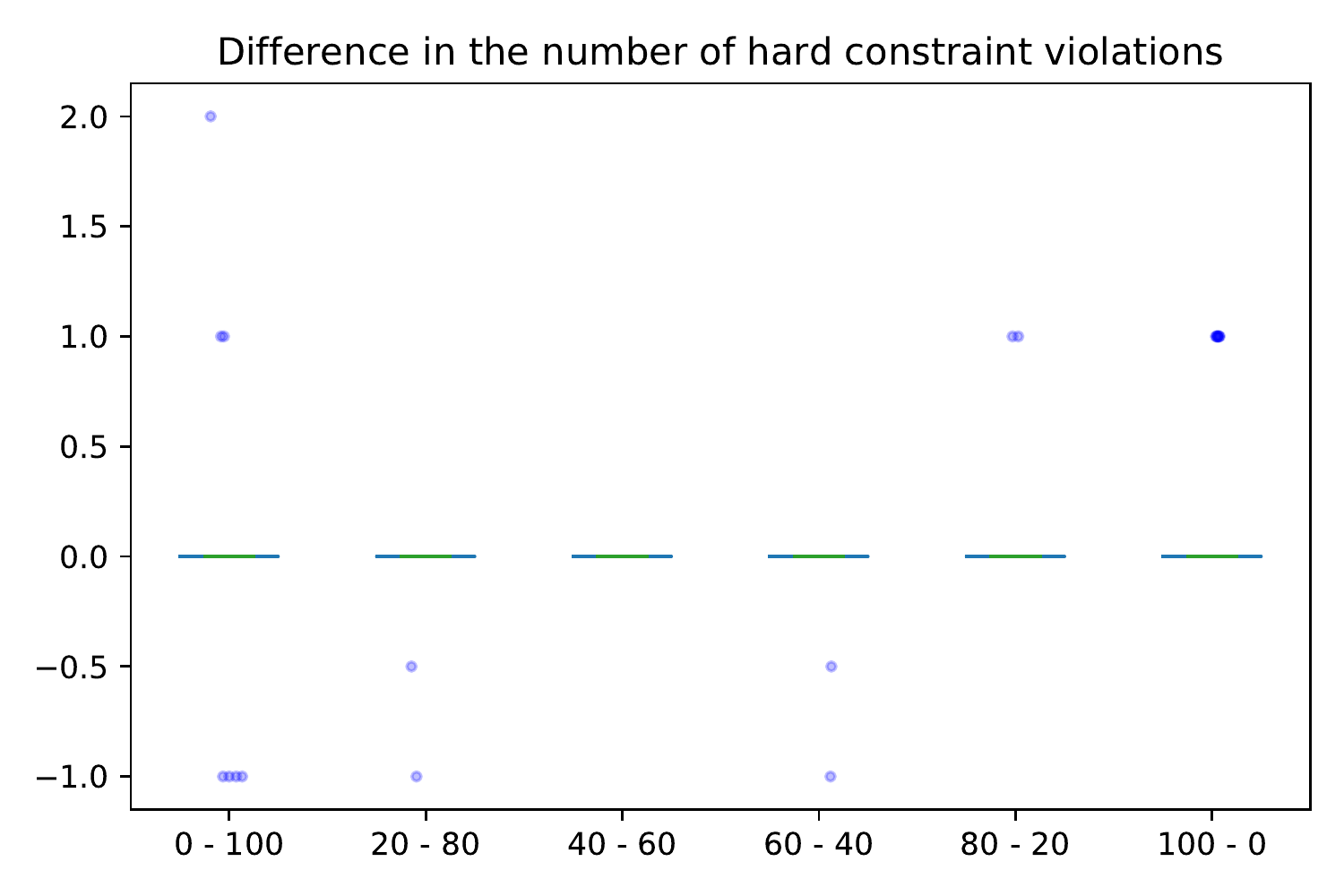}
    \end{subfigure}
    
	\caption{Results of the experiment regarding neighborhood weightings. All the results are comparisons per instance against the schedule '40 - 60'. The first number of each label is the weight of the move neighborhood and the second the weight of the swap neighborhood.}
	\label{fig:experiment-neighborhood-weighting}
\end{figure}

To sum up, the tested neighborhood weightings between '20 - 80' and '60 - 40' are equally good, therefore it is logical to assume that the untested weightings in between are good as well.
Using one of the other weightings leads to worse results, especially in the case where only the move neighborhood is used.

\subsubsection{Optimality gap of metaheuristic solutions}

An interesting question during the evaluation of metaheuristic techniques is by how far the solutions deviate from the optimal solution.
As stated previously, \gls{mip} solvers measure this property in terms of the optimality gap, which is calculated by taking one minus a lower bound divided by the objective value.
In the following we assess how large the optimality gap of the metaheuristic solutions is which can be done by using the lower bounds obtained though \gls{mip}.
However, as only a small fraction of the instances could be solved well enough that this kind of evaluation makes sense, we present afterwards also an analysis based on the randomly generated instances with known optimal solutions.

\begin{enumerate}
	\item \textbf{Optimality gap for small instances}:
	We want to analyze the optimality gap of the solutions produced by our metaheuristic approaches.
	Therefore, we use the best dual bound found by the \gls{mip} solver.
	In order to get even better bounds, we executed the solver again with a time limit of 10 hours and used for each instance the best available bound.
	We restrict this evaluation to the instances in $R_1$ and $R_4$ because they are the only sets where the instances are small enough so that we could obtain mostly good bounds.
	Furthermore, we select the subset of instances whose optimality gap of the best \gls{mip} solution is below $10\%$ because we can only assume safely that the dual bound is good if the gap is small.
	This step eliminates 8 instances of $R_1$ ($30\%$) and 4 of $R_4$ ($40\%$), which means that 25 instances remain for our evaluation.
	By using the best bound the optimality gap for each of the metaheuristic approaches is calculated on the selected instances.
	
	Figure~\ref{fig:mip_r1_r4_optimality_gap} shows the optimality gap for each instance in the reduced set for the greedy heuristic, \gls{vnd}, Simulated Annealing and \gls{mip}.
	For solutions, which are not valid because of constraint violations, no mark is shown.
	The figure conveys, that the solutions found by the greedy construction heuristic have gaps between $10\%$ and $25\%$ and a considerable number of instances is not solved to feasibility at all. \gls{vnd} already achieves drastic improvements by solving all instances but four with a maximum gap of $10\%$ and mostly around $5\%$.
	Simulated Annealing is clearly the best metaheuristic for this restricted set of small instances, as it produced always valid solutions which have a similar optimality gap as the \gls{mip} solutions and in several cases even better.
	The gap is most almost always below $5\%$ percent and with an average of about $3\%$.
	
	The resulting numbers state how large the relative difference between metaheuristic and optimal solutions is \emph{at most}
	\footnote{The gap is calculated using the best lower bound which was proven by the MIP solver. As most of the solutions were not proven optimal, the bounds are most probably smaller than the optimal objective value and thus the calculated gap an upper bound of the actual gap.}.
	This result is interesting because it proves that our Simulated Annealing approach solves the majority of the small instances (which includes most of the realistic instances) extremely well.
	On average the solutions are at most $3\%$ above the optimal one.
	
	\begin{figure}
		\centering
		\includegraphics[width=\linewidth]{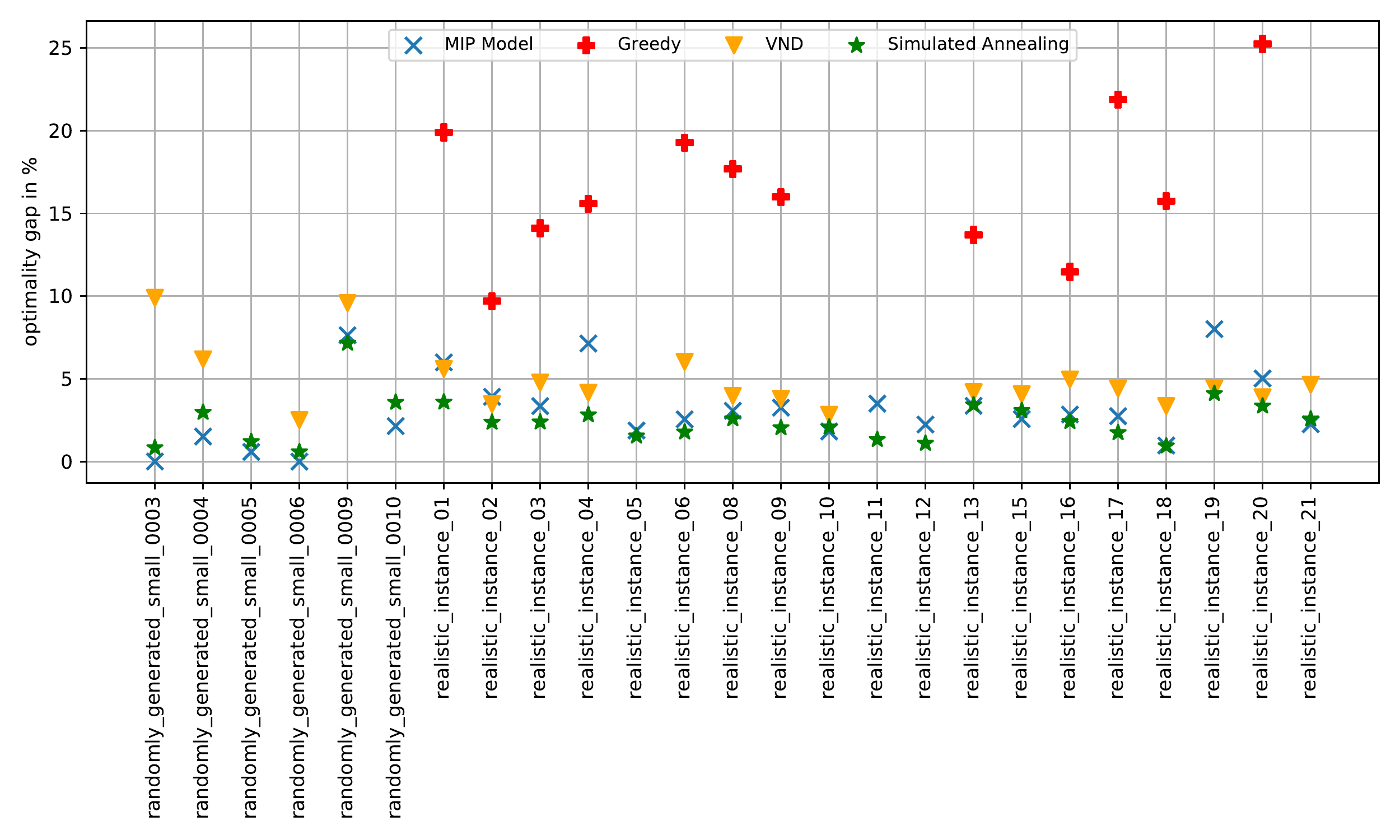}
		\caption{Comparison of the optimality gap of valid solutions obtained though Greedy, \gls{vnd}, Simulated Annealing and \gls{mip}. Marks are missing for solutions that violate constraints. The gap is computed using the dual bounds of \gls{mip}. The evaluation contains all instances of $R_1$ and $R_4$ which could be solved with an optimality gap of $10\%$ or lower using \gls{mip}}
		\label{fig:mip_r1_r4_optimality_gap}
	\end{figure}
	
	\item \textbf{Comparison on large randomly perfect instances $\mathbf{R_2}$}:
	The instances in $R_2$ have been constructed in a way that the optimal solutions are known and have the objective value $0$.
	This enables us draw some conclusions about the optimality gap also for larger instances than those for which we can obtain bounds using \gls{mip}.
	However, the optimality gap as defined above cannot be computed for the instances in $R_2$ because the best known bound is zero which would lead to a division by zero.
	Instead, we can say something similar by looking at the objective function.
	The first objective function component $g_1$, defined in Equation~\eqref{obj1_abs} states, informally speaking, the gap to a hypothetical perfectly leveled solution averaged over all periods.
	In case of the instances $R_2$ this hypothetical solution actually exists and the value $g_1$ can be interpreted as the average percentage by which planned demand for each period exceeds or falls short of the target.
	The same argument holds for $g_2$, defined in Equation~\eqref{obj2_abs}, which states average percentage by which planned demand for each period exceeds or falls short of the target, averaged over the different products.
	The objective component $g_3$, defined in Equation~\eqref{obj3_abs}, can be interpreted as the percentage of actual priority inversions measured against the theoretical maximum number of inversions $\frac{k \cdot (k-1)}{2}$.
	
	Table~\ref{tab:r2_results} shows the values of $g_1$, $g_2$ and $g_3$ multiplied by 100, so that they can be interpreted as percentages, as explained above for each each algorithm and for each instance of $R_2$ where a valid solution has been reached.
	The fourth column reports the percentage of valid solutions.
	We can see that Simulated Annealing and \gls{vnd} reach negligible mean deviations for the first two objectives.
	However, also the greedy heuristic produces levelings which deviate from the target by only $2\%$ on average, so we can assume that this part of the task is not very hard for this instance set.
	With respect to the priority objective the results leave some more room for improvements.
	The greedy heuristic reaches the best value here (at the cost of fewer valid solutions and a worse leveling), followed by \gls{vnd} and Simulated Annealing.
	The best total results are clearly reached by \gls{vnd} for this instance set.
	It is not entirely clear why the results on $R_2$ differ so much from the results which are obtained on the other parts of the test set but we suspect that it might have something to do with that the greedy heuristic finds very good initial solutions on this instance set.
	That might help \gls{vnd} more than Simulated Annealing because that latter starts with a random search which destroys some of the good structure which was already there.
	
	\begin{table}
		\centering
		\caption{Results for metaheuristic methods on $R_2$. The values of the objective function components $g_1$, $g_2$ and $g_3$ multiplied by 100 so that they can be interpreted as percentages for each each algorithm for each instance of $R_2$ where a valid solution has been reached. The rightmost column states for how many percent of the solutions each algorithm reached a valid solution.}
		\label{tab:r2_results}
		\begin{tabular}{@{}lrrrr@{}}
			\toprule
			{} & $g_1 (\%)$ & $g_2 (\%)$ & $g_3 (\%)$ &  valid (\%) \\
			\midrule
			\textbf{Greedy                      } &  1.78 &  2.25 &  1.78 &  78.00 \\
			\textbf{VND                         } &  0.04 &  0.34 &  2.62 &  96.00 \\
			\textbf{Simulated Annealing (median)} &  0.32 &  0.47 &  7.95 &  92.00 \\		
			\textbf{Optimum                     } &  0.00 &  0.00 &  0.00 & 100.00 \\
			\bottomrule
		\end{tabular}
		\end{table}
\end{enumerate}

The above analysis of the optimality gap on small, realistic instances and the larger perfectly solvable ones revealed that our metaheuristic methods produce solutions which are only few percentage points away from the optimum.
When taking also the comparison of the different metaheuristics on the whole test set into account, which was presented in the previous section, we can conclude that Simulated Annealing is the overall best of our algorithms for the \gls{plp} because it can solve small instances in an excellent way and still scales to the largest instances which we have.

\section{Conclusion}
\label{sec:conclusion}

We introduced a new combinatorial optimization problem in the area of production planning, which concerns the assignment of orders to production periods. 
Thereby a number of production capacity constraints need to be fulfilled and a work balancing objective as well as the prioritization of the orders must be optimized.

We started with modeling the problem formally and provided a proof of \NP-hardness.
Then we turned towards solution methods and introduced a new \gls{mip} formulation.
Finally we investigated local search methods and defined two neighborhood structures for the problem, which we evaluated using \gls{vnd} and Simulated Annealing.

The main results of this work are:
\begin{itemize}
	\item The \gls{plp} is \NP-hard, which was shown by an \NP-completeness proof of the associated decision problem via a reduction from Bin Packing.
	
	\item The proposed \gls{mip} model solves instances up to medium size.
	The complexity of solving grows with every dimension of the problem, but most notably with the number of orders $k$.
	For instances with less than 250 orders we can expect either a feasible solution or the proof of infeasibility within one hour of time, while we cannot count on finding any solution for instances with $k \geq 300$.
	
	\item With Simulated Annealing very good solutions can be obtained within five minutes.
	The real-life instances can all be solved well and for most of them we can show though the use of dual bounds that the solutions are within $3\%$ of optimality on average.
	Experiments based on the set of instances with perfectly leveled solutions indicate that Simulated Annealing is capable of providing really good solutions also for much larger instances.
	
	\item Another metaheuristic option studied in this work is \gls{vnd}.
	While the objective value and the number of hard constraint violations is a bit higher on average compared to Simulated Annealing, it still finds good solutions most of the time.
	The instances with perfectly leveled solutions $R_2$ are even solved better by \gls{vnd} than Simulated Annealing.
	
	\item An experiment regarding the weighting of the two neighborhoods in Simulated Annealing showed that it is clearly advantageous to use both neighborhoods instead of either of them alone.
	The best weighting between the move and the swap neighborhood is between '20 - 80' and '60 - 40'.
\end{itemize}

Future work could include improvements of the \gls{mip} model, e.g.\ by decomposition based techniques, so that more instances can be solved optimally or better bounds are obtained.
That would open up ways to assess the quality of metaheuristic solutions of large instance and to see more clearly where they fall short.

\subsubsection*{Acknowledgments}
The financial support by the Austrian Federal Ministry for Digital and Economic Affairs and the National Foundation for Research, Technology and Development is gratefully acknowledged.
Further, we would like to thank Christoph Mrkvicka (MCP GmbH) for his help regarding problem formulation.

\bibliographystyle{plain}  
\bibliography{lit}

\end{document}

%% file: definitions.tex
\usepackage{todonotes}
\usepackage{amsthm, amsmath}
\usepackage{tabularx, booktabs, paralist}
\usepackage{diagbox}
\usepackage{multirow}
\usepackage{subcaption}

\usepackage[ruled,linesnumbered]{algorithm2e} 
\LinesNumbered

\usepackage{nicefrac}

\usepackage[acronym,toc]{glossaries}

\newtheorem{theorem}{Theorem}

\newcommand{\cprob}[3]{
\begin{center}
{\small
\begin{tabularx}{0.98\columnwidth}{ll}
\toprule
\multicolumn{2}{c}{\textsc{#1 (Decision problem)}} \\
\midrule
\textbf{Instance}:   & \parbox[t]{0.8\columnwidth}{#2\vspace*{1mm}}  \\
\textbf{Question}:   & \parbox[t]{0.8\columnwidth}{#3\vspace*{.5mm}} \\
\bottomrule
\end{tabularx}
}
\end{center}
} 
 
\newcommand{\NP}{$\mathcal{NP}$}

\newacronym{plp}{PLP}{Production Leveling Problem}
\newacronym{bacp}{BACP}{Balanced Academic Curriculum Problem}
\newacronym{salb}{SALB}{Simple Assembly Line Balancing}
\newacronym{tps}{TPS}{Toyota Production System}
\newacronym{mip}{MIP}{Mixed Integer Programming}
\newacronym{cp}{CP}{Constraint Programming}
\newacronym{sos1}{SOS1}{Specially Ordered Set Type 1}
\newacronym{vnd}{VND}{Variable Neighborhood Descent}
\newacronym{vns}{VNS}{Variable Neighborhood Search}
\newacronym{smac}{SMAC}{Sequential Model-based Algorithm Configuration}

%% file: methods/mip-model.tex
\begin{align}
\textsf{min}\hspace{1cm} & a_1 g_1 + a_2 g_2 + a_3 g_3 && \label{ex8:1}\\
\textsf{s.t.}\hspace{1cm} & \sum_{i \in N} x_{ij} = 1							& j \in K 				  \label{eq:mip_one_x}\\
						  & \sum_{i \in N} i \cdot x_{ij} = y_j 				& j \in K 			      \label{eq:mip_link_xy}\\
						  & y_i - y_j \leq (n-1) z_{ij}							& i, j \in K \enspace\vert\enspace p_i > p_j \label{eq:mip_link_yz}\\
						  & \sum_{j \in K} d_j x_{ij} + s_i^+ - s_i^- = d^*   	& i \in N         \label{eq:mip_slack_constr}\\
						  & \sum_{j \in K | t_j = t} d_j x_{ij}  + s_{it}^+ - s_{it}^- = d_t^*    & i \in N, \enspace t \in M \label{eq:mip_slack_p_constr}\\
						  & d^* + s_i^+ \leq c      							& i \in N 			         \label{eq:mip_cmax_total}\\
						  & d_t^* + s_{it}^+ \leq c_t   						& i \in N, \enspace t \in M \label{eq:mip_cmax[p]}\\
						  & y_i \leq y_j                                        & i, j \in S, S \subseteq {K} \enspace\vert\enspace p_i \geq p_j, d_i = d_j, t_i = t_j \label{eq:mip_symmetry} \\
						  & \sum_{t \in M} (s_{it}^- - s_{it}^+) = s_i^- - s_i^+   & i \in N \label{eq:mip_link_slacks}
\end{align}

%% file: technical-report.bbl
\begin{thebibliography}{10}

\bibitem{azizoglu_workload_2018}
Meral Azizoğlu and Sadullah İmat.
\newblock Workload smoothing in simple assembly line balancing.
\newblock {\em Computers \& Operations Research}, 89:51--57, January 2018.

\bibitem{boysen_classification_2007}
Nils Boysen, Malte Fliedner, and Armin Scholl.
\newblock A classification of assembly line balancing problems.
\newblock {\em European Journal of Operational Research}, 183(2):674 -- 693,
  December 2007.

\bibitem{boysen_product_2009}
Nils Boysen, Malte Fliedner, and Armin Scholl.
\newblock The product rate variation problem and its relevance in real world
  mixed-model assembly lines.
\newblock {\em European Journal of Operational Research}, 197(2):818--824,
  September 2009.

\bibitem{chiarandini_balanced_2012}
Marco Chiarandini, Luca Di~Gaspero, Stefano Gualandi, and Andrea Schaerf.
\newblock The balanced academic curriculum problem revisited.
\newblock {\em Journal of Heuristics}, 18(1):119--148, February 2012.

\bibitem{di_gaspero_hybrid_2008}
Luca Di~Gaspero and Andrea Schaerf.
\newblock Hybrid {Local} {Search} {Techniques} for the {Generalized} {Balanced}
  {Academic} {Curriculum} {Problem}.
\newblock In Maria~J. Blesa, Christian Blum, Carlos Cotta, Antonio~J.
  Fernández, José~E. Gallardo, Andrea Roli, and Michael Sampels, editors,
  {\em Hybrid {Metaheuristics}, 5th {International} {Workshop}, {HM} 2008,
  {Málaga}, {Spain}, {October} 8-9, 2008. {Proceedings}}, volume 5296 of {\em
  Lecture {Notes} in {Computer} {Science}}, pages 146--157. Springer, 2008.

\bibitem{gurobi}
LLC Gurobi~Optimization.
\newblock Gurobi optimizer reference manual, 2019.

\bibitem{hansen_variable_2010}
Pierre Hansen, Nenad Mladenović, Jack Brimberg, and José A.~Moreno Pérez.
\newblock Variable {Neighborhood} {Search}.
\newblock In Michel Gendreau and Jean-Yves Potvin, editors, {\em Handbook of
  {Metaheuristics}}, International {Series} in {Operations} {Research} \&
  {Management} {Science}, pages 61--86. Springer US, Boston, MA, 2010.

\bibitem{kirkpatrick_optimization_1983}
S.~Kirkpatrick, C.~D. Gelatt, and M.~P. Vecchi.
\newblock Optimization by {Simulated} {Annealing}.
\newblock {\em Science}, 220(4598):671--680, May 1983.

\bibitem{kubiak_minimizing_1993}
Wieslaw Kubiak.
\newblock Minimizing variation of production rates in just-in-time systems: {A}
  survey.
\newblock {\em European Journal of Operational Research}, 66(3):259--271, May
  1993.

\bibitem{lindauer_smac_2019}
Marius Lindauer, Katharina Eggensperger, Matthias Feurer, Stefan Falkner,
  André Biedenkapp, and Frank Hutter.
\newblock {SMAC} v3: {Algorithm} {Configuration} in {Python}, January 2019.

\bibitem{mullinax_assigning_2002}
C.~Mullinax and M.~Lawley.
\newblock Assigning patients to nurses in neonatal intensive care.
\newblock {\em Journal of the Operational Research Society}, 53(1):25--35,
  January 2002.

\bibitem{ohno_toyota_1998}
Taiichi Ohno and C.~B. Rosen.
\newblock {\em Toyota production system: beyond large-scale production}.
\newblock Productivity Press, Portland, OR, 1998.

\bibitem{punnakitikashem_stochastic_2013}
Prattana Punnakitikashem, Jay~M. Rosenberber, and Deborah~F. Buckley-Behan.
\newblock A stochastic programming approach for integrated nurse staffing and
  assignment.
\newblock {\em IIE Transactions}, 45(10):1059--1076, October 2013.

\bibitem{schaus_solving_2009}
Pierre Schaus.
\newblock {\em Solving balancing and bin-packing problems with constraint
  programming}.
\newblock PhD thesis, UCL - Université Catholique de Louvain, 2009.

\bibitem{schaus_deviation_2007}
Pierre Schaus, Yves Deville, Pierre Dupont, and Jean-Charles Régin.
\newblock The {Deviation} {Constraint}.
\newblock In Pascal~Van Hentenryck and Laurence~A. Wolsey, editors, {\em
  Integration of {AI} and {OR} {Techniques} in {Constraint} {Programming} for
  {Combinatorial} {Optimization} {Problems}, 4th {International} {Conference},
  {CPAIOR} 2007, {Brussels}, {Belgium}, {May} 23-26, 2007, {Proceedings}},
  volume 4510 of {\em Lecture {Notes} in {Computer} {Science}}, pages 260--274.
  Springer, May 2007.

\bibitem{schaus_scalable_2009}
Pierre Schaus, Pascal~Van Hentenryck, and Jean-Charles Régin.
\newblock Scalable {Load} {Balancing} in {Nurse} to {Patient} {Assignment}
  {Problems}.
\newblock In Willem Jan~van Hoeve and John~N. Hooker, editors, {\em Integration
  of {AI} and {OR} {Techniques} in {Constraint} {Programming} for
  {Combinatorial} {Optimization} {Problems}, 6th {International} {Conference},
  {CPAIOR} 2009, {Pittsburgh}, {PA}, {USA}, {May} 27-31, 2009, {Proceedings}},
  volume 5547 of {\em Lecture {Notes} in {Computer} {Science}}, pages 248--262.
  Springer, May 2009.

\bibitem{vazirani_approximation_2003}
Vijay~V. Vazirani.
\newblock {\em Approximation {Algorithms}}.
\newblock Springer-Verlag Berlin Heidelberg, 2003.

\bibitem{warner_scheduling_1976}
D.~Michael Warner.
\newblock Scheduling {Nursing} {Personnel} {According} to {Nursing}
  {Preference}: {A} {Mathematical} {Programming} {Approach}.
\newblock {\em Operations Research}, 24(5):842--856, October 1976.

\end{thebibliography}
